\documentclass[final,5p,times,twocolumn,preprint]{elsarticle}

\usepackage{epsfig} 
\usepackage{graphicx}
\usepackage{caption}
\usepackage{float}
\usepackage{epsfig}
\usepackage{url}
\usepackage{verbatim}
\usepackage{amssymb}
\usepackage{fancyhdr}
\usepackage{subfigure}
\usepackage{booktabs}
\usepackage{xcolor}
\usepackage{textcomp}
\usepackage{algorithm}
\usepackage{multirow}
\usepackage{algorithm}
\usepackage{algorithmic}
\usepackage{array}
\usepackage{subfig}
\usepackage{stmaryrd}
\usepackage{amsfonts,latexsym,amssymb}
\usepackage{amsmath}
\usepackage{amsthm}
\usepackage{multirow}

\input{tensor.math}

\newdefinition{definition}{Definition}
\newtheorem{theorem}{Theorem}
\newtheorem{remark}{Remark}
\newtheorem{lemma}[theorem]{Lemma}
\newproof{pf}{Proof}

\journal{Elsevier}

\begin{document}

\begin{frontmatter}

\title{Towards Efficient and Accurate Approximation: Tensor Decomposition Based on Randomized Block Krylov Iteration}

\author[a]{Yichun Qiu}
\ead{ycqiu@gdut.edu.cn}
\author[a,b]{Weijun Sun \corref{*}}
\ead{gdutswj@gdut.edu.cn}
\author[a,c]{Guoxu Zhou}
\ead{gx.zhou@gdut.edu.cn}
\fntext[*]{* Corresponding author.}
\author[a,d,e]{Qibin Zhao}
\ead{qibin.zhao@riken.jp}

\address[a]{School of Automation, Guangdong University of
	Technology, Guangzhou 510006, China}
\address[b]{Guangdong-Hong Kong-Macao Joint Laboratory for Smart Discrete Manufacturing,  Guangzhou 510006, China}
\address[c]{Key Laboratory of Intelligent Detection and The Internet of Things in Manufacturing, Ministry of Education, Guangzhou 510006, China}
\address[d]{Joint International Research Laboratory of Intelligent Information Processing and System Integration of IoT, Ministry of Education, Guangzhou 510006, China}
\address[e]{Center for Advanced Intelligence 	Project (AIP), RIKEN, Tokyo, 103-0027, Japan}

  \begin{abstract}
    Efficient and accurate low-rank approximation (LRA) methods are of great significance for large-scale data analysis. Randomized tensor decompositions have emerged as powerful tools to meet this need, but most existing methods perform poorly in the presence of noise interference. 
    Inspired by the remarkable performance of randomized block Krylov iteration (rBKI) in reducing the effect of tail singular values, this work designs an rBKI-based Tucker decomposition (rBKI-TK) for accurate approximation, together with a hierarchical tensor ring decomposition based on rBKI-TK for efficient compression of large-scale data. Besides, the error bound between the deterministic LRA and the randomized LRA is studied. Numerical experiences demonstrate the efficiency, accuracy and scalability of the proposed methods in both data compression and denoising. 
  \end{abstract}


\end{frontmatter}

\section{Introduction}
\label{intro}
Large-scale multi-way data, such as color images, video sequences and EEG signals, are ubiquitous in scientific computing and numerical applications \cite{Kolda2009,Tensors2015}. 
Tensor decomposition models provide a natural and efficient way to represent, store and compute such data \cite{Tucker1966, TTD2011, Zhou2012CPD, TRD2016, FCTN2021}. 
Since observation data are often noisy, incomplete and redundant, low-rank tensor decomposition plays an increasingly important role in many machine learning tasks, such as noise reduction \cite{Chen2022,Du2022}, data completion \cite{YuTR2021,Shao2022,Wu2022} and data mining \cite{He2019,Wangcls2021,3wayclus2021}. 
However, directly using traditional methods (e.g., singular value decomposition (SVD) \cite{HOSVD2000}, alternating least squares (ALS) \cite{HOOI2000}) to compute low-rank approximation (LRA) is quite space and time consuming, especially for large-scale problems. 

Accordingly, randomized sketching techniques offer the possibility of eliminating this limitation.
The basic idea of the randomized sketching is to capture a significant subset or subspace of the data to produce an approximation with acceptable accuracy
but less computational cost. 
Specifically, randomized sampling and randomized projection are frequently employed as sketching techniques to efficiently accelerate LRA \cite{randskt2015,randskt2017}.
In many large-scale problems, randomized methods outperform the classical competitors in terms of efficiency \cite{randAlgs2011,iterLRA2017}.
Without loss of generality, in the paper, randomized sketching techniques are divided into two main categories: non-iterative sketching and iterative sketching. 
Non-iterative sketching tends to sketch the input data in one-pass, whereas iterative sketching tends to sketch through power iterations.

In parallel with the development of randomized sketching techniques, a number of randomized tensor decomposition methods have emerged as powerful tools for large-scale data analysis. 
In non-iterative sketching techniques, for example, there are Tucker decomposition (TKD) \cite{che2019,Feng2020,2iQR2020}, tensor-train decomposition (TTD) \cite{che2019}, tensor ring decomposition (TRD) \cite{Yuan2019TR} and hierarchical Tucker + Canonical Polyadic decomposition (CPD) \cite{2iQR2020} designed based on Gaussian random projection, TKD \cite{csSVD2018} and CPD \cite{Wang2015} designed based on CountSketch sampling, as well as CPD \cite{lsSVD2022} and TRD \cite{Malik2022} designed based on leverage score sampling.
In iterative sketching techniques, for example, there are TKD with two ALS-like iterations \cite{2iQR2020}, TKD with random range finder (RRF)-based initialization and two ALS-like iterations \cite{3iSVD2021}, CPD \cite{randCP2020} and TKD \cite{piSVD2021} with power iterations designed based on Gaussian random projection, TKD \cite{Minster2020} and TRD \cite{rrQR2021} designed based on rank revealing projection, and TKD designed based on Krylov subspace projection \cite{K-Vec-Ten2013,BKS2022,Eldn2022}. 

However, real-world data are often corrupted by varying degrees of noise, which leads to a heavy tail for singular values of the data.
In this case, non-iterative sketching techniques can neither eliminate the effects of tail noise nor achieve a good approximation.
On the contrary, iterative sketching techniques have the potential to address this issue by using power iterations.
Related works have indicated the strength of iterative sketching techniques over non-iterative sketching techniques in reducing the impact of tail singular values.
Therefore, this work aims to search accurate approximation for noisy tensors by employing more efficient iterative sketching method.

Recently, randomized block Krylov iteration (rBKI) \cite{rBKI2015} has raised concerns in the LRA of matrices \cite{rBKI2021,Wang2021,rBKI2022}. 
Compared to traditional Krylov-type methods, rBKI has two advantages: on the one hand, convergence performance is improved by using random sketch as initialization, and on the other hand only fewer iterations are needed to obtain the same guaranteed accuracy.
Moreover, existing Krylov-type TKD methods either lead to inefficiencies and slow convergence or are restricted to specific types of tensor (e.g., symmetric, sparse and 3D-tensors).
Therefore, in this work, tensor decomposition methods based on rBKI are proposed for more general observation tensors (with attributes of dense, noisy or more than three dimensions), which is not a simple extension but a further exploration with the following contributions.

$\bullet$ Given the superiority in constructing accurate subspace and reducing the effect of tail singular values, an rBKI-based TKD (rBKI-TK) is designed for LRA of noisy tensors, which not only offers the possibility to explore internal structural information from multiple perspectives, but also yields near-optimal approximation with the best known theoretical efficiency.
Besides, further investigation is provided and it is verified that instead of using full input data to construct the Krylov subspace, using partial data (even 10\%) can yield acceptable results with reduced complexity.

$\bullet$ To address large-scale tensor decomposition, a hierarchical TRD method is developed by leveraging rBKI-TK compression, which can effectively reduce the cost of storage and computation, as well as the guaranteed accuracy. The proposed method is of great importance to enable deep learning to be applied to large-scale tasks. 

$\bullet$  Considering that it is necessary to analyze the approximation error between randomized and deterministic methods.
Theoretical analysis is provided in the paper to discuss how the sketching affects the accuracy of LRA, which has not been done in the existing works (to our knowledge).

$\bullet$ A relatively comprehensive comparison of existing randomized tensor decompositions, including non-iterative and iterative series, is carried out in the simulation study. 
It is verified that the proposed rBKI-TK can achieve excellent denoising performance in both synthetic and real-world data. Moreover, in terms of time consumption, the rBKI-TK-TR also shows its great strength over traditional TR-ALS.

The following sections are organized as follows. 
In Section \ref{sec:Preliminary}, some frequently used notations and definitions are briefly overviewed, and a theoretical analysis of the randomized LRA problem is presented. In Section \ref{sec:rBKI-TK}, the main algorithms are developed based on rBKI, including a rBKI-based TKD (rBKI-TK) for accurate approximation and a hierarchical TRD based on rBKI-TK for efficient compression of large-scale tensor. In Section \ref{sec:Sim-Dis}, simulations on synthetic and real-world data are conducted to verified the superiority of the proposed methods, and analysis are provided to present new insights. Finally, conclusions and future directions are depicted in Section \ref{sec:conclusion}.

\section{Preliminaries}
\label{sec:Preliminary}
\subsection{Notations and Definitions}
Throughout the paper, the terms ``order'' or ``mode'' are used to interchangeably describe a tensor data.
For example, a third-order tensor can be unfolded into a mode-$n$ matricization in each mode with $n=1,2,3$.
Besides, a tensor set by [N=3, I=20, R=3] represents a third-order tensor with multi-dimension =[20,20,20] and multilinear-rank =[3,3,3].
For better description, some frequently used notations are listed in Table \ref{tab:notations}.
Besides, two basic tensor decomposition models studied in this work that are described below.

\begin{table}[htbp] \scriptsize
  \renewcommand\arraystretch{1.2}
  \caption{Notations \& Description}
  \label{tab:notations}
  \centerline{
  \begin{tabular}{ >{\hfill}p{.15\textwidth} | p{.22\textwidth}  }
  \hline \hline 
  $\Real$  & Set of real numbers.  \\
  $\mathbf{x}, \mat{X}, \tensor{X}$ & Vector, matrix, and tensor.\\
  $\mat{X}^{\top}$ &  The transpose of matrix \mat{X}.\\
  \tenmat{X} &  The mode-$n$ matricization of tensor \tensor{X}. \\
  $\tensor{X}(i,j,k)$ &  The ($i,j,k$)th-entry of tensor \tensor{X}. \\
  $\kkp$ & The Kronecker product. \\
  $\times_n$ & The mode-$n$ product.\\
  $\|\mathbf{X}\|_F = \sqrt{\mathbf{X}^{\top}\mathbf{X}}$ & Frobenius norm of matrix \mat{X}.\\
  \hline \hline
  \end{tabular}}
  \end{table}

  \begin{definition}
    Tucker decomposition (TKD) of an $N$th-order tensor $\tensor{X}\in \Real^{I_1 \times I_2 \cdots \times I_N}$ is expressed as an $N$th-order core tensor interacting with $N$ factor matrices. Suppose the multilinear rank of $\tensor{X}$ is $(R_1,R_2,\ldots,R_N)$, then it has the following mathematical expression, 
    \begin{equation}
    \label{eq:TKD}
    \begin{split}
    \tensor{X}&= \tensor{G}\ttmn[1]{U}\ttmn[2]{U}\cdots\ttmn[N]{U}\\
    &\defeq \compactTucker{G}{U}, 
    \end{split}
    \end{equation}
    where $\matn{U}\in\Real^{I_n\times R_n}$ is the mode-$n$ factor matrix, $\tensor{G}\in\Real^{R_1\times R_2\cdots\times R_N}$ is the core tensor reflecting the connections between the factor matrices.
  \end{definition}
  
Typically, when computing a TKD, the mode-n matricization on \tensor{X} as $\tenmat{X} \in\Real^{I_n \times \prod_{p\neq n}I_p}$, $n \in \Natural_N$ is necessary to be performed first and then written in a form of matrix product.  
  \begin{equation}
  \label{eq:unTKD}
    \min \frob{\tenmat{X}-\matn{U}\tenmat{G}\left(\bigkkp\nolimits_{p\neq n}\matn[p]{U}\right)^{\top}}^2, \quad n = 1,2, \cdots,N,
  \end{equation}
where $\matn{U} \in\Real^{I_n\times R_n}$, $R_n \le min(I_n, \prod_{p\neq n}I_p)$, and $\tenmat{G} \left(\bigkkp\nolimits_{p\neq n}\matn[p]{U}\right)^{\top} \in\Real^{R_n \times \prod_{p\neq n}I_p}$ can be computed by SVD \cite{HOSVD2000} or ALS \cite{HOOI2000} algorithms. 

\begin{definition}
  Tensor ring decomposition (TRD) of an $N$th-order tensor $\tensor{X}\in\Real^{I_1\times I_2\cdots\times I_N}$ is expressed as a cyclic multilinear product of $N$ third-order core tensors that can be shifted circularly. Given a TR-rank of ($R_1\times R_2\cdots\times R_N$), the TRD operator in denoted by $\tensor{X} \defeq \mathfrak{R} \tenfactors{ \tensor{X}_1,\tensor{X}_2, \cdots, \tensor{X}_N }$, and the element-wise TRD is expressed below.
  \begin{equation}
    \label{eq:TR}
    \begin{split}
    & \quad \tensor{X}(i_1, i_2, \cdots, i_{_N}) \\
    & = \sum_{r_1 = 1}^{R_1} \sum_{r_2 = 1}^{R_2} \cdots \sum_{r_{_N} = 1}^{R_{_N}}  \tensor{X}_1^{(i_1)}(r_{_N},r_1) \tensor{X}_2^{(i_2)}(r_1,r_2) \cdots \tensor{X}_N^{(i_{_N})}(r_{_{N-1}},r_{_N}) ,
    \end{split}
  \end{equation}
  \end{definition}
  where $\tensor{X}_n \in \Real^{R_{n-1} \times I_n \times R_n} (n = 1,2, \cdots,N)$ are the TR factors (or TR cores) of $\tensor{X}$, and $\tensor{X}_n^{(i_n)} \in \Real^{R_{n-1} \times R_n}$ denotes the $i_n$-th slice of the $n$-th factor. 

Note that, each factor $\tensor{X}_n$ contains one dimension-mode ($I_n$ in mode-2) and two rank-modes ($R_{n-1}$ in mode-1 and $R_n$ in mode-3). 
Besides, adjoint factors $\tensor{X}_k$ and $\tensor{X}_{k+1}$ have a shared rank to ensure the contraction operation between them, e.g., $R_N=R_1$. 
Similarly, when computing TRD, the expression (Eq. \ref{eq:TR}) can be rewritten with a mode-$n$ matricization form as,
\begin{equation}
  \label{eq:unTR}
      \min \frob{\tenmat{X}-\mathbf{X}_{(n,2)} (\mathbf{X}_{(p \neq n,2)})^{\top}}^2, \quad n = 1,2, \cdots,N,
\end{equation}
where $\mathbf{X}_{(n,2)} \in \Real^{I_n  \times R_{n-1} R_n}$ and $\mathbf{X}_{(p \neq n,2)} \in \Real^{ \prod_{p\neq n} I_p \times R_{n-1} R_n  }$  denotes the mode-2 matricization of the $n$-th factor and the rest $N-1$ factors, respectively.

\subsection{Low-rank Approximation}
Low-rank approximation (LRA) is widely used in many machine learning tasks, such as data completion, noise reduction and network compression.
In this subsection, the LRA problem with respect to Frobenius norm, and the error bound for $(1+\epsilon)$ quasi-optimal approximation is described as follows.
\begin{definition}
  \label{df:LRA}
  For any matrix $ \mathbf{X} \in \Real^{n \times d}$ with specific rank($ \mathbf{X}$) = $r \le \min(n,d)$, the LRA problem is to compute a rank-$k$ ($k \le r$) approximation $ \tilde{\mathbf{X}}_k$ whose Frobenius norm error is comparable to the optimal rank-$k$ approximation $\mathbf{X}_k$, and the ideal error $\epsilon$ is very small,
  \begin{equation}
    \label{eq:F_errb}
    \|\mathbf{X}-\tilde{\mathbf{X}}_k\|_F^2 \le (1+\epsilon) \|\mathbf{X}-\mathbf{X}_k\|_F^2.
  \end{equation}
\end{definition}

Generally, the optimal rank-$k$ approximation $\mathbf{X}_k$ can be obtained from the truncated SVD of $\mathbf{X}$, which is known as the Eckart-Young Theorem: 
\begin{lemma}
  \label{lm:E-Young}
  Assume that $[\mathbf{U}, \mathbf{\Sigma}, \mathbf{V} ]= \text{svd}(\mathbf{X})$, and $\mathbf{X}_k = \mathbf{U}_k \mathbf{\Sigma}_k \mathbf{V}_k^{\top} = \mathbf{U}_k \mathbf{U}_k^{\top} \mathbf{X}$, where $\mathbf{U}_k$ denotes the left singular vectors corresponding to the top-$k$ singular values in $\mathbf{\Sigma}$.
  Then, for any matrix $\tilde{\mathbf{X}}_k$ with rank-$k$, the minimal error is achieved with $\mathbf{X}_k$. That is, $\min_{\tilde{\mathbf{X}}_k} \| \mathbf{X} - \tilde{\mathbf{X}}_k \|_F= \| \mathbf{X} - \mathbf{X}_k \|_F$.
\end{lemma}
 
However, using SVD to compute the LRA is prohibitive for large-scale data, so randomized SVD have emerged to address this limitation. 
By utilizing variant randomized sketching techniques ( e.g., random sampling and random projection), randomized SVD aims to achieve satisfactory approximation accuracy while significantly reducing the cost of storage and computation \cite{randAlgs2011,randskt2015,randskt2017}.  
Here, LRA based on any of the randomized sketching techniques is termed as randomized LRA, which has two main stages: 1) construct a low-dimensional space and draw a sketch that reflects the importance of data; 2) project the data onto the sketching space and compute a SVD for LRA.

Specifically, given a matrix $ \mathbf{X} \in \Real^{n \times d}$, any random sketching matrix $\mathbf{\Omega} \in \Real^{d \times s}$ with a sketch-size $ k \le s \le \min(n,d)$ can be used to draw a randomized sketch from $\mathbf{X}$, i.e., $\mathbf{Y} = \mathbf{X} \mathbf{\Omega} \in \Real^{n \times s}$, and then compute the orthonormal basis  $\mathbf{Q} \in \Real^{n \times s}$ of $\mathbf{Y}$ to project the data onto the sketching space $\mathbf{Q}^{\top} \mathbf{X} \in \Real^{s \times d}$.

\begin{lemma}
  The Frobenius norm error bound for the randomized LRA holds with high probability that  
  \begin{equation}
    \label{eq:sk_F_errb1}
    \|\mathbf{X}-\mathbf{Q} (\mathbf{Q} ^{\top} \mathbf{X})_k \|_F^2 \le (1+\epsilon) \|\mathbf{X}-\mathbf{U}_k \mathbf{U}_k ^{\top} \mathbf{X} \|_F^2.
  \end{equation}
\end{lemma}
As $\mathbf{Q} \in \Real^{n \times s}, s \ge k$ denotes a $s$-dimension subspace of $\Real^n$, $\mathbf{Q}(\mathbf{Q} ^{\top} \mathbf{X})_k$ spans the $(1+\epsilon)$ rank-$k$ approximation to $\mathbf{X}$ by using truncated SVD. Clearly, sketching can further speed up the LRA when $s > k$. 
More details about the theoretical guarantees and proofs can be found in the work \cite{skttool2014}.

\begin{theorem}
  \label{th:main1}
  In the subspace of the orthonormal basis $\mathbf{Q}\in \Real^{n \times s}$, $(\mathbf{Q} ^{\top} \mathbf{X})_k$ denotes the optimal approximation to $\mathbf{U}_k ^{\top} \mathbf{X}$, and $\mathbf{Q}(\mathbf{Q} ^{\top} \mathbf{X})_k$ gives the quasi-optimal to $\mathbf{X}$. 
\end{theorem}

\begin{proof}
  Let $\mathbf{Z}$ be an arbitrary projection of $\mathbf{X}$ with the same rank and size as $\mathbf{U}_k^{\top} \mathbf{X}$, then we have
  \begin{equation}
    \label{eq:proof1}
    \begin{split}
      \| \mathbf{X} - \mathbf{Q}(\mathbf{Q} ^{\top} \mathbf{X})_k\|_F^2 
      & = \|\mathbf{X}-\mathbf{U}_k \mathbf{U}_k ^{\top} \mathbf{X} \|_F^2 + \|\mathbf{U}_k \mathbf{U}_k ^{\top} \mathbf{X}-\mathbf{Q} (\mathbf{Q} ^{\top} \mathbf{X})_k \|_F^2 \\ 
      & \le \|\mathbf{X}-\mathbf{U}_k \mathbf{U}_k ^{\top} \mathbf{X} \|_F^2 +  \|\mathbf{U}_k \mathbf{U}_k ^{\top} \mathbf{X}-\mathbf{Q} \mathbf{Z} \|_F^2 \\ 
      & = \| \mathbf{X}-\mathbf{Q} \mathbf{Z} \|_F^2.
    \end{split}
    \end{equation}
\end{proof}
The inequality holds when $(\mathbf{Q} ^{\top} \mathbf{X})_k$ is the optimal approximation to $\mathbf{U}_k ^{\top} \mathbf{X}$.
It is verified that $\mathbf{Q}(\mathbf{Q} ^{\top} \mathbf{X})_k$ is the quasi-optimal rank-$k$ approximation to $\mathbf{X}$ in the subspace $\mathbf{Q}$.

\subsection{Low-multilinear-rank Approximation and Error Bounds}

Multilinear sketching (or tensor sketching) is an extension of matrix sketching to higher order tensor, which is also an essential part of the randomized higher-order SVD (HOSVD).
The main purpose of this subsection is to introduce tensor sketching and to analyze the error bound between deterministic LRA and randomized LRA.

Given an $N$th-order tensor $\tensor{X} \in \Real^{I_1,I_2,\dots,I_N}$, let $\mathbf{\Omega}^{(n)} \in \Real^{I_n \times S_n}, R_n \le S_n \le \min(I_n,\prod_{p\neq n}I_p), n = 1,2, \cdots,N$ be the random matrix on mode-$n$, and $\mathbf{Q}^{(n)} \in \Real^{I_n \times S_n}$ be the orthonormal basis of the randomized initialization $\mathbf{X}_{(n)} \mathbf{\Omega}^{(n)}$.

\begin{lemma}
  \label{lm:TensorSketch}
  The Frobenius norm error bound for multilinear randomized LRA is denoted as
  \begin{equation}
    \label{eq:sk_F_errb2}
    \begin{split}
    & \|\mathbf{X}_{(n)}-\mathbf{Q}^{(n)} (\mathbf{Q} ^{(n)\top} \mathbf{X}_{(n)})_{R_n} \|_F^2 \\
    & \le (1+\epsilon) \|\mathbf{X}_{(n)}-\mathbf{U}^{(n)} \mathbf{U} ^{(n)\top} \mathbf{X}_{(n)} \|_F^2, \quad n = 1,2, \cdots,N, 
    \end{split}
  \end{equation}
  where the $(1+\epsilon)$ rank-$R_n$ approximation can be obtained by performing truncated HOSVD on the projection $\mathbf{Z}^{(n)} = \mathbf{Q} ^{(n)\top} \mathbf{X}_{(n)}$.
\end{lemma}

\begin{remark}
  It is equivalently to get $\mathbf{Z}^{(n)}$ from $\mathbf{Q}^{(n)\top} (\mathbf{X}_{(n)} \mathbf{X}_{(n)}^{\top})\mathbf{Q}^{(n)} \in \Real^{S_n \times S_n}$.
  In specific, let $\tilde{\mathbf{U}} \mathbf{\Lambda} \tilde{\mathbf{U}}^{\top} = \text{evd}(\mathbf{Z}^{(n)})$, then $ (\mathbf{Q}^{(n)}\tilde{\mathbf{U}}_{R_n} ) 
  (\mathbf{Q}^{(n)}\tilde{\mathbf{U}}_{R_n})^{\top} \mathbf{X}_{(n)} =\mathbf{Q}^{(n)} (\mathbf{Q} ^{(n)\top} \mathbf{X}_{(n)})_{R_n}$. 
\end{remark}

Suppose that $R_n$ = $S_n$, $\tensor{G} = \tensor{X} \times_1 \mathbf{Q} ^{(1)\top} \times_2  \mathbf{Q} ^{(2)\top} \cdots \times_N  \mathbf{Q} ^{(N)\top}$ is the multilinear sketched core tensor of \tensor{X}, and $ \tilde{\tensor{X}}_k$ with $ k=(R_1,R_2,\ldots,R_N)$ is the multilinear randomized LRA of tensor \tensor{X}, then Eq. \ref{eq:sk_F_errb2} can be rewritten in a multilinear product (or folding) form 
\begin{equation}
  \label{eq:tensorsketch}
  \begin{split}
  \tilde{\tensor{X}}_k 
  & = \tensor{G} \times_1  \mathbf{Q}^{(1)} \times_2  \mathbf{Q}^{(2)} \cdots \times_N \mathbf{Q}^{(N)} \\
  & = \tensor{X} \times_1  (\mathbf{Q}^{(1)} \mathbf{Q} ^{(1)\top}) \times_2  (\mathbf{Q}^{(2)} \mathbf{Q} ^{(2)\top}) \cdots \times_N  (\mathbf{Q}^{(N)} \mathbf{Q} ^{(N)\top}). 
  \end{split}
\end{equation}

\begin{remark}
According to the Pythagorean theorem, 
\begin{equation}
  \label{eq:Pythagorean}
  \| \tensor{X} - \tilde{\tensor{X}}_k \|_F^2
   \le \sum_{n=1}^N 
  \| \tensor{X} - \tensor{X} \times_n  (\mathbf{Q}^{(n)} \mathbf{Q} ^{(n)\top}) \|_F^2. 
\end{equation}
\end{remark}

Since we perform LRA on the sketch of \tensor{X} with $R_n \le S_n$, it is necessary to discuss how the sketching affects the accuracy of LRA.
Assume that $\tilde{\tensor{X}} = \tensor{X} \times_1  (\mathbf{Q}^{(1)} \mathbf{Q} ^{(1)\top}) \times_2  (\mathbf{Q}^{(2)} \mathbf{Q} ^{(2)\top}) \cdots \times_N  (\mathbf{Q}^{(N)} \mathbf{Q} ^{(N)\top})$ is the multilinear sketch of \tensor{X}, the error bound between deterministic LRA and randomized LRA is provided as follows.
\begin{theorem}
  \label{pr:ErrIneq}
  Let $\tensor{X}_k$ and $\tilde{\tensor{X}}_k$ be the optimal rank-$k$ approximation of the input tensor \tensor{X} and the sketched tensor $\tilde{\tensor{X}}$, respectively. If $0 \le \| \tensor{X} - \tilde{\tensor{X}}\|_F \le \alpha$ and $ \| \tensor{X} - \tensor{X}_k\|_F = \beta$, then we have
  \begin{equation}
    \label{eq:ErrIneq}
        \beta \le \| \tensor{X} - \tilde{\tensor{X}}_k\|_F \le 2\alpha+\beta.
  \end{equation}
\end{theorem}

\begin{proof}
  \begin{equation}
    \label{eq:proof}
    \begin{split}
      \beta \le \| \tensor{X} - \tilde{\tensor{X}}_k\|_F 
      & \le \| \tensor{X} - \tilde{\tensor{X}}\|_F + \| \tilde{\tensor{X}} - \tilde{\tensor{X}}_k\|_F \\ 
      & \le \alpha + \| \tilde{\tensor{X}} - \tensor{X}_k\|_F \\ 
      & \le \alpha + \| \tilde{\tensor{X}} - \tensor{X}\|_F + \| \tensor{X} - \tensor{X}_k\|_F \\
      & \le 2\alpha + \beta.
    \end{split}
    \end{equation}
\end{proof}
Suppose that $\tilde{\tensor{X}}$ is sketched without loss (i.e., $\alpha = 0$), $\tilde{\tensor{X}}_k$ can achieve the optimal approximation as $\tensor{X}_k$.  
In other word, since  $\sqrt{1+\epsilon} - 1 = 2 \alpha / \beta$, an accurate sketch ($\alpha \approx 0$) provides possibility to obtain $(1+\epsilon) $ approximation  with a very small ideal error $\epsilon$. 
Overall, it is verified that the randomized LRA can achieve satisfactory approximation to the original data, which is also borne out in the simulation section.

\section{Main Algorithms}
\label{sec:rBKI-TK}

\subsection{Randomized Block Krylov Iteration (rBKI)}

Krylov-type methods have a long history in matrix approximation, which can be traced back to works \cite{lanczos1950,lanczos1965}.
Traditional Krylov methods tended to capture the key information of data by simply constructing a Krylov subspace in a vector-by-vector paradigm, i.e., 
\begin{equation}
\mathbf{K} \doteq [\mathbf{v}, \mathbf{X}\mathbf{v}, \mathbf{X}^2\mathbf{v}, \cdots, \mathbf{X}^{q-1}\mathbf{v}],
\end{equation} 
where $\mathbf{X}\in \Real^{n \times d}$ is the input and $\mathbf{v}\in \Real^{n}$ is a starting vector.
To improve the efficiency, block Krylov method was proposed to construct a Krylov subspace in a block-by-block paradigm, i.e., 
\begin{equation}
\mathbf{K} \doteq [\mathbf{V}, \mathbf{X}\mathbf{V}, \mathbf{X}^2\mathbf{V}, \cdots, \mathbf{X}^{q-1}\mathbf{V}],
\end{equation}
where $\mathbf{V}\in \Real^{n \times s}$ is an initial matrix.  

However, when directly extended to tensor cases, the simple Krylov method with randomized tarting vector failed to obtain satisfactory accuracy in some certain cases \cite{K-Vec-Ten2013}, and the block Krylov method with random initial matrix $\mathbf{V}$ may result in poor convergence and was only restricted to sparse or symmetric case \cite{BKS2022,Eldn2022}.
To eliminate these limitations, randomized block Krylov iteration (rBKI) \cite{rBKI2015} was proposed to construct a Krylov subspace using a randomized sketch $\mathbf{W} = \mathbf{X\Omega}$ instead of an initial matrix $\mathbf{V}$, i.e., 
\begin{equation}
  \label{eq:BKspace}
\mathbf{K} \doteq [\mathbf{W}, (\mathbf{X}\mathbf{X}^{\top})\mathbf{W},\cdots, (\mathbf{X}\mathbf{X}^{\top})^q\mathbf{W}],
\end{equation}
where $\mathbf{\Omega} \in \Real^{d \times s}$ is a randomized sketching matrix. 

Although non-iterative randomized methods (e.g., RRF) and power iteration method ($\mathbf{K} \doteq  (\mathbf{X}\mathbf{X}^{\top})^q\mathbf{W}$) are popular in LRA, rBKI performed better in reducing the effect of tail singular values. 
As shown in Fig. \ref{fig:denoisePc}, given an [N=3, I=10, R=5] tensor as input, the clean data was drawn from $\mathcal{N}(0,1)$ and the noisy data was corrupted by SNR=0dB Gaussian random noise. Heat-maps of singular value matrices of clean data, noisy data, and projections based on variant methods were presented. Clearly, rBKI performed better in eliminating tail singular values than non-iterative and power iteration methods.
With a more accurate projection subspace, rBKI facilitates higher accuracy approximation.
\begin{figure}[htbp]
  \centerline{
      \includegraphics[width=1\linewidth,height=0.17\linewidth]{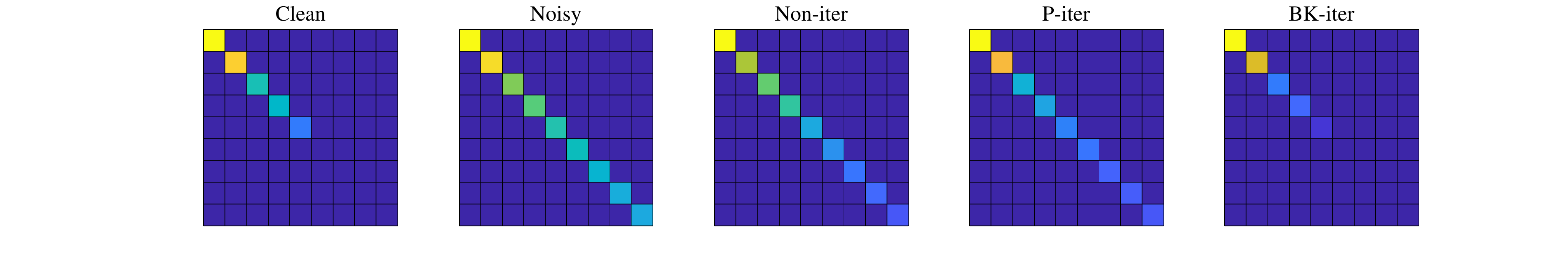}
  }
  \caption{Comparison of non-iterative (Non-iter) method, power iteration (P-iter) and randomized block Krylov iteration (BK-iter) in removing tail singular values of an [N=3, I=10, R=5] tensor.}
  \label{fig:denoisePc}
  \end{figure}

Besides, since rBKI constructs the subspace through a lower $q$-degree polynomials (\ref{eq:BKspace}), it requires only $q = \varTheta (\text{log} d / \sqrt{\epsilon} )$ iterations to achieve the quasi-optimal approximation \footnote{ In Eq. \ref{eq:F_errb}, the approximation that satisfies error ($1+\epsilon$) is called quasi-optimal approximation.} to an $n \times d$ matrix, which is more efficient than the power iteration that requires $q = \varTheta (\text{log} d / \epsilon )$ iterations (Theorem 1 in \cite{rBKI2015}).
Empirically, $q = 2 $ is sufficient to achieve satisfactory approximate accuracy.
Particularly, when the data have a rapid singular value decay, rBKI can offer faster convergence performance, which is validated in the simulation section.
A detailed description of tensor sketching based on rBKI is presented in Algorithm\ref{alg:RBKI}. 

\begin{algorithm}[htbp]
  \caption{Tensor sketching based on rBKI}
  \label{alg:RBKI}
  \begin{algorithmic}[1]
   \REQUIRE $\tenmat{X}\in\Real^{I_n\times \prod_{p\neq n}I_p }$, the mode-n sketch-size $S_n$, and error $\epsilon \in (0,1)$.
   \ENSURE $[\matn{Q}, \sim] = \text{qr} (\matn{K}$, 0). \# economical QR. 
  \STATE Let $q = \Theta (\frac{log \prod_{p\neq n}I_p }{\sqrt{\epsilon}} )$, $\matn{\Omega}\in\Real^{\prod_{p\neq n}I_p\times S_n}$ is a randomized sketching matrix.
  \STATE Draw a sketch $\matn{W} = \tenmat{X}\matn{\Omega}$.
  \STATE Construct the block Krylov subspace $\matn{K} = [\matn{W}, (\tenmat{X}\tenmat{X}{}^{\top})\matn{W},\cdots, (\tenmat{X}\tenmat{X}{}^{\top})^q\matn{W}]$.
  \end{algorithmic}
  \end{algorithm} 

Moreover, the time complexity of several typical compared methods was presented in Table \ref{tab:complexity} for comparison.
Given a matrix $\mathbf{X} \in \Real^{n \times d}$ with sketch-size $s$, according to the cost of each stage, the deterministic SVD gave a runtime of $\mathcal{O}(n^2d)$, while the GR-SVD gave that of $\mathcal{O} (nds)$.
Since the number of iteration for rBKI and power iteration method are $q = \varTheta (\text{log} d / \sqrt{\epsilon} )$ and $q = \varTheta (\text{log} d / \epsilon )$, the corresponding runtime are $\mathcal{O}(nds\log d / \sqrt{\epsilon}  + ns^2 \log^2 d/ \epsilon + s^3 \log^3 d/ \epsilon^{\frac{3}{2}})$ and $\mathcal{O}(nds\log d / \epsilon + ns^2 \log d/ \epsilon)$, respectively.

\begin{table}[htb!] \scriptsize
  \renewcommand\arraystretch{1.2} 
  \caption{Comparison of time complexity between the deterministic SVD and randomized SVD based on RRF (GR-SVD), power iteration (GRpi-SVD) and rBKI (rBKI-SVD), on an $(n \times d)$ matrix with sketch-size $s$.}
  \label{tab:complexity}
  \centering  
  \begin{tabular}{ l | c c c c }
\hline \hline	
Methods  & SVD & GR-SVD &GRpi-SVD &rBKI-SVD \\
\cline{1-5}
sketching $\mathbf{W}$    & /      & $\mathcal{O} (nds)$ & $\mathcal{O} (nds)$        & $\mathcal{O} (nds)$        \\
constructing $\mathbf{K}$   & /      & /     & $\mathcal{O} (ns^2q+ndsq)$ & $\mathcal{O} (ns^2q+ndsq)$ \\
finding $\mathbf{Q}$    & /      & $\mathcal{O}(ns^2)$              & $\mathcal{O}(ns^2)$                  & $\mathcal{O}(n(sq)^2)$                  \\
computing $\mathbf{Z}$   & /      & $\mathcal{O}(nds)$               & $\mathcal{O}(nds+ns^2)$                 & $\mathcal{O}(nd(sq)+nd(sq)^2)$          \\
factorization & $\mathcal{O}(n^2d)$ & $\mathcal{O}(s^2d)$              & $\mathcal{O}(s^3)$                      & $\mathcal{O}((sq)^3)$                   \\
\hline \hline
  \end{tabular}
\end{table}

In addition, when the input data have a large number of columns, the time complexity of constructing a block Krylov subspace can be further reduced by selecting partial columns (e.g., uniform sampling). 
Specifically, for a matrix $\mathbf{X} \in \Real^{n \times d}$ ($d$ is quite large and $ n \ll d$ ), a shorter matrix $\mathbf{Y} \in \Real^{n \times m}$ ($ m = \omega d$, and $\omega$ is a sampling rate) can be used to reduce the time complexity of constructing a block Krylov subspace. 
It is verified that acceptable accuracy can be achieved even with a small sampling rate of $\omega = 0.1$, as detailed in the simulation section (see Fig. \ref{fig:samprate}).

\subsection{Tucker Decomposition Based on rBKI (rBKI-TK)}

Tucker decomposition (TKD) is a natural extension of SVD to higher order tensors, also known as higher order SVD (HOSVD) \cite{HOSVD2000}. 
As HOSVD is time and space consuming for large-scale tensor, this work aims to explore more efficient TKD based on randomized HOSVD by utilizing randomized sketching techniques.

Inspired by the implementation of rBKI algorithm in LRA of matrices, a randomized HOSVD based on rBKI algorithm is proposed for computing TKD, which is detailed in Algorithm \ref{alg:rBKI-TKD}. Specifically, step 2 illustrates how rBKI can be integrated into a randomized HOSVD framework to capture a more important subspace of the input data.
As a result, a fast and accurate TKD can be achieved by the rBKI-based randomized HOSVD.

\begin{algorithm}[htbp]
  \caption{The rBKI-based TKD (rBKI-TK)}
  \label{alg:rBKI-TKD}
  \begin{algorithmic}[1]
   \REQUIRE $\tensor{X}\in\Real^{I_1\times I_2 \cdots \times I_N}$, N-th sketch-sizes $(S_1,S_2,\ldots,S_N)$ and the multilinear rank $k = (R_1,R_2,\ldots,R_N)$. 
   \ENSURE $\tilde{\tensor{X}}_k = \compactTucker{G}{U}$.
  \FOR{$n=1,2,\ldots,N$}
  \STATE Compute the orthonormal basis via Algorithm.\ref{alg:RBKI} $\matn{Q} = \text{rBKI}(\tenmat{X}, S_n, \epsilon)$.
  \STATE Project the data onto the sketch subspace $\matn{Z} = \matn{Q}{}^{\top}\tenmat{X}\tenmat{X}{}^{\top}\matn{Q}$.
  \STATE Compute [$\matn{\tilde{U}},\sim,\sim] = \text{svd} (\matn{Z})$.  \# economical SVD
  \STATE Compute $\matn{U} = \matn{Q} \matn{\tilde{U}}(:, 1:R_n)$.
  \STATE Compute the sketched core tensor $\tensor{G}\from\tensor{X}\ttmn{U}{}^{\top}$.
  \ENDFOR
  \end{algorithmic}
  \end{algorithm}

\subsection{TRD with Prior rBKI-TK Compression}

With a well-designed structure, tensor ring decomposition (TRD) \cite{TRD2016} usually served as an efficient tensor compression model for large-scale data analysis.
While TRD can significantly reduce dimensionality, it is still inefficient when directly applied to large-scale data analysis. 
Accordingly, it is necessary to employ an efficient tensor compression prior to TRD to reduce the complexity. 

In this subsection, a hierarchical TRD model with prior rBKI-TK compression is proposed for large-scale data compression, which is detailed in Algorithm \ref{alg:rBKI-TK-TR}.  
Note that, such hierarchical framework can be extended to other tensor decompositions, such as CPD and TTD.
For better understanding, a tensor network diagram of TKD, TRD, hierarchical TK-TR decomposition and contraction operations among them are shown in Fig.  \ref{fig:TK-TR}.

\begin{algorithm}[htbp]
  \caption{TRD with prior rBKI-TK compression (rBKI-TK-TR) }
  \label{alg:rBKI-TK-TR}
  \begin{algorithmic}[1]
   \REQUIRE $\tensor{X}\in\Real^{I_1\times I_2 \cdots \times I_N}$, compress-sizes $(S_1,S_2,\ldots,S_N)$ and TR-rank $k = (R_1,R_2,\ldots,R_N)$. 
   \ENSURE $\tilde{\tensor{X}}_k = \mathfrak{R} \tenfactors{ \tensor{X}_1,\tensor{X}_2, \cdots, \tensor{X}_N }$.
  \STATE The prior rBKI-TK compression $\compactTucker{G}{U} = \text{rBKI-TK} (\tensor{X}, S_n, n = 1, 2, \cdots,N)$.
  \FOR{$n=1,2,\ldots,N$}
  \STATE Compute $\tensor{G}_n = \text{TR-ALS}(\tensor{G},R_n) \in \Real^{R_{n-1} \times S_n \times R_n}$. \# or using TR-SVD \cite{TRD2016}.
  \STATE Contraction operation $\tensor{X}_n = \tensor{G}_n \times_n \matn{U} \in \Real^{R_{n-1} \times I_n \times R_n}$.
  \ENDFOR
  \end{algorithmic}
  \end{algorithm} 

\begin{figure*}[htb!]
\centerline{
    \includegraphics[width=1\linewidth,height=0.5\linewidth]{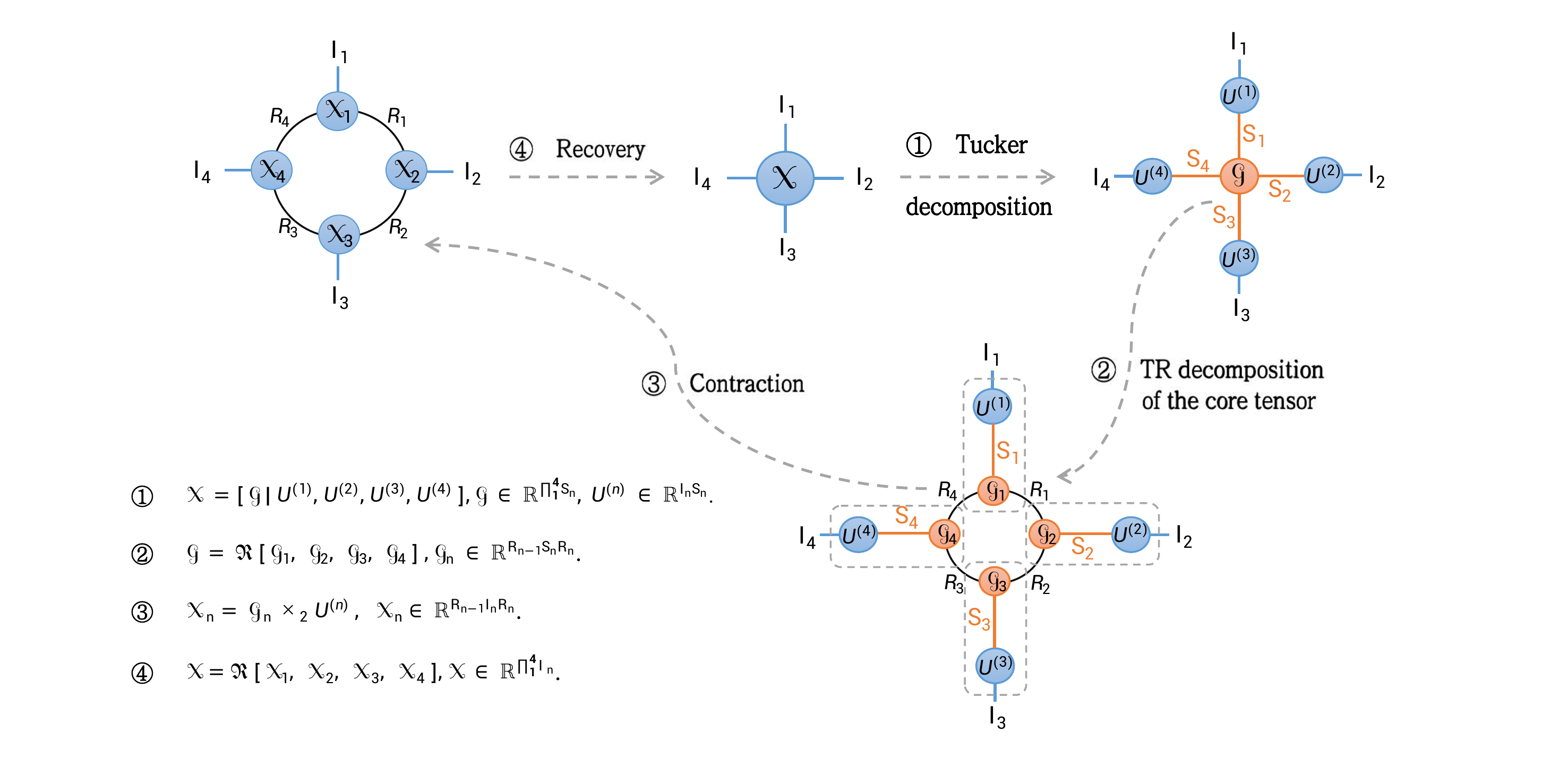}
}
\caption{A tensor network diagram of TKD, TRD, hierarchical TK-TR decomposition and contraction operations for a fourth-order tensor.}
\label{fig:TK-TR}
\end{figure*}

\section{Simulations and Analysis}
\label{sec:Sim-Dis}

In this section, efficient denoising of noisy tensors were considered. 
The proposed method was compared with a number of randomized sketching methods.
All the simulations were performed on a computer with Intel Core i5 CPU (3 GHz) and 8 GB memory, running 64-bit Windows 10, in 64-bit MATLAB2020r.
Notably, all results were the average of 10 Monte Carlo runs.
For all the methods, the number of maximum iterations was set to 20, and a stopping criteria was set by $|\text{Fit}(\tensor{X}, \tilde{\tensor{X}}_k^t) - \text{Fit}(\tensor{X}, \tilde{\tensor{X}}_k^{t-1})| \le 10^{-3}$, where $\tilde{\tensor{X}}_k^t$ is a rank-k approximation of $\tensor{X}$ after $t$-th iteration.
Running time in seconds was used for efficiency validation and the following evaluation metrics were employed to measure the accuracy of approximation and denoising. 

$\bullet$ Fit(\%) $ = (1-\|\tensor{X}-\tilde{\tensor{X}}_k \|_F / \|\tensor{X} \|_F) \times 100\%$, Fitting rate, where $\tensor{X}$ and $\tilde{\tensor{X}}_k$ is the original tensor and the sketched rank-$k$ approximation, respectively. A larger value indicates higher recovery accuracy.

$\bullet$ RErr $ = \| \tensor{X}-\tilde{\tensor{X}}_k \|_F / \|\tensor{X} - \tensor{X}_k\|_F$, Relative Error, where $\tensor{X}_K$ is the optimal rank-$k$ approximation obtained by truncated HOSVD. ERrr = 1 means that $\tilde{\tensor{X}}_k$ approximates as well as $\tensor{X}_k$, with larger ratios indicates worse results.

$\bullet$ PSNR(dB) $= 10 \log_{10} (255^2 / \text{MSE})$, Peak-Signal-Noise-Ratio, where MSE is the Mean-Square-Error denoted by $\| \tensor{X} - \tilde{\tensor{X}}_k \|_F / \text{num}(\tensor{X})$, and $\text{num}(\cdot)$ denotes the number of components. A larger value indicates better denoising performance.

\subsection{Simulations on rBKI-TK}
In this subsection, the proposed rBKI-TK model was compared with a number of randomized methods. 
Since the proposed rBKI sketches in an iterative strategy, several iterative sketching techniques were chosen for comparison, including Gaussian randomized HOSVD models based on power iterations (GRpi-SVD) \cite{piSVD2021}, one RRF-initialization+two-ALS-like iterations (GR3i-SVD) \cite{3iSVD2021}, two-ALS-like iterations (GR2i-SVD) \cite{2iQR2020} and rank revealing algorithm (GRrr-SVD) \cite{rrQR2021}.
Besides, non-iterative sketching methods were also employed to demonstrate the superiority of the iterative series over the non-iterative series, including existing randomized HOSVD models based on Gaussian random range finder (GR-SVD) \cite{2iQR2020},  leverage score sampling (LS-SVD) \cite{lsSVD2022}, as well as randomized HOSVD models we generated based on CountSketch sampling (CS-SVD), importance sampling (IS-SVD), and subsampled randomized Hadamard transform (SRHT-SVD). 
Moreover, the truncated HOSVD (T-SVD) was adopted as a benchmark method for approximate quality assessment. 

Typical types of synthetic data and real-world video data were used to evaluate the approximation and denoising performance of the compared methods.
Besides, in all simulations on rBKI-TK, the sketch size was set by $S_n = R_n+1/\gamma$, where the overlapping parameter $\gamma$ was set to 0.2, as  suggested in \cite{randAlgs2011}.
For consistency and fairness, the sampling rate $\omega$ is set to 1 for the proposed method in all simulations unless specified.

\subsubsection{Synthetic Data}
The generation of synthetic tensor data followed $\tensor{X}+\lambda \tensor{N}$, where the original data \tensor{X} was corrupted by a Gaussian random noise term $\lambda \tensor{N}$, and $\lambda$ was a parameter for tuning the noise level of  $\tensor{N}$ via Signal-to-Noise-Ratio (SNR)\footnote{SNR = 0dB means volume(Data) = volume(Noise), SNR $<$ 0dB means volume(Data) $<$ volume(Noise), and vice versa.}
\begin{equation}
  \label{eq:SNR}
\text{SNR(dB)} = 20 \text{log}_{10} (\|\tensor{X}\|_F / \|\lambda \tensor{N}\|_F).
\end{equation}
The tensor \tensor{X} was generated in two ways: Gaussian random and power functional. Compared to the Gaussian random data, the power functional data often has a tail of smaller singular values (i.e., a rapid singular value decay), which is prone to good low-rank approximation results. 
For better understanding, a heatmap of the singular value matrix of Gaussian random data and power functional data were shown in Fig. \ref{fig:heatmap}. 
As shown, in the Gaussian random case, $\sigma_1(\mathbf{X}_{(1)}) / \sigma_2(\mathbf{X}_{(1)}) \approx 1.6$, while in the power functional case, $\sigma_1(\mathbf{X}_{(1)}) / \sigma_2(\mathbf{X}_{(1)}) \approx 5$.

$\bullet$ {\bf Gaussian random data.} In this case, a third-order Tucker tensor $\tensor{X}$ was given by \eqref{eq:TKD}, where the entries of latent factor matrices $\matn{U}\in\Real^{I\times R}, n=1,2,3$ and core tensor $\tensor{G}\in\Real^{R_1\times R_2\cdots\times R_N}$ were drawn from $\mathcal{N}(0,1)$.   
To measure the performance of denoising, the additional Gaussian noise tensor $\tensor{N}$ was generated with different SNR = [-10,-5,5] dB. 
The experimental results were presented in Table \ref{tab:noisydata}.

$\bullet$ {\bf Power Functional data.}
In this case, a third-order tensor $\tensor{X}$ was given by
\begin{equation}
  \label{eq:function}
 x_{i,j,k} = \frac{1}{\sqrt[p]{i^p + j^p +k^p}}, 
\end{equation}
where $p$ is a power parameter controlling the singular values of $\tensor{X}$ (p=10 in this simulation). 
To evaluate the scalability of rBKI-TK, $\tensor{X}$ was generate in different size and corrupted by the same level (SNR = 5 dB) Gaussian noise.
The experimental results were presented in Table \ref{tab:fundata}.

\begin{figure}[htb!]
  \centerline{
    \subfigure[Gaussian random data]{
      \includegraphics[width=0.22\textwidth,height=0.20\textwidth]{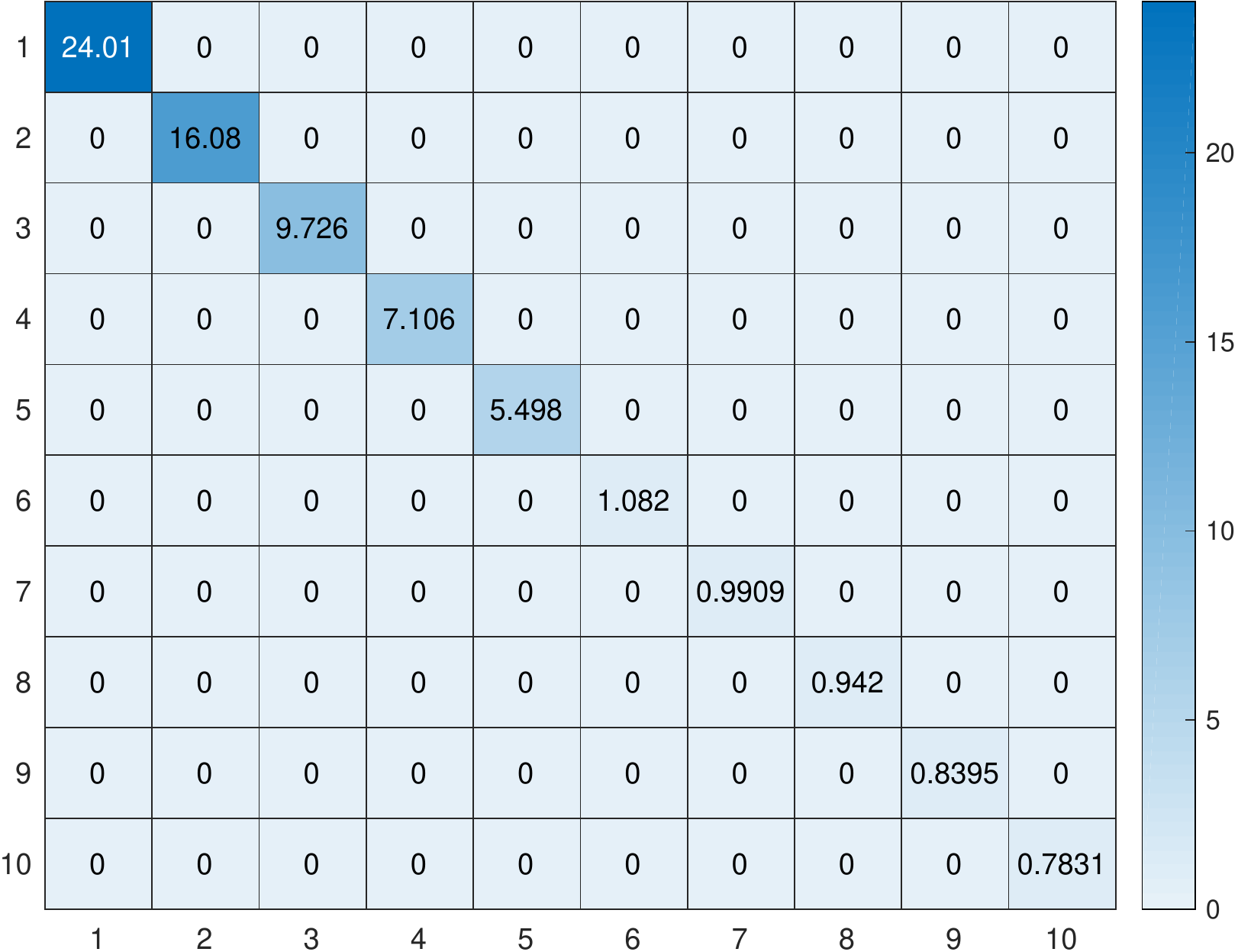}}\hfil
      \subfigure[Power functional data]{
      \includegraphics[width=0.22\textwidth,height=0.20\textwidth]{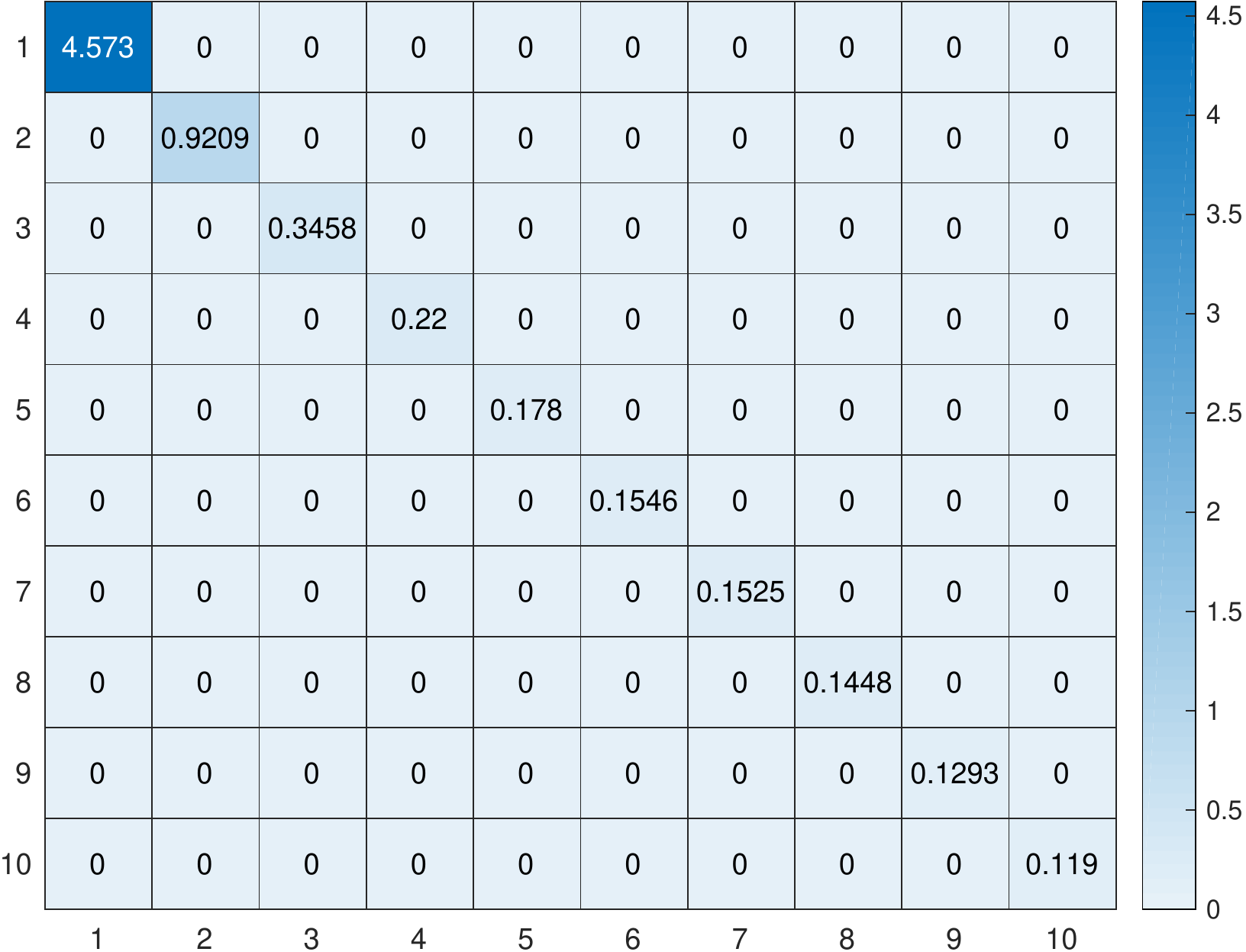}}
      }
  \caption{For an (N=3, I=10, R=5, SNR=20dB) tensor $\tensor{X}$ in (a): Gaussian random format and in (b): power functional format with rapid singular value decay, a heatmap of the singular value matrix of the mode-1 matricization.}
  \label{fig:heatmap}
  \end{figure}

  \begin{table*}[!t] \scriptsize
  \renewcommand\arraystretch{1.2} 
  \caption{Comparison in terms of RErr, Fit(\%) and Time(s) of variant methods on an [N=3, I=200, R=10] Gaussian random tensor with different noise ratio settings, SNR = [-10,-5,5] dB. Compared to the deterministic T-SVD, our method achieve the best results among the randomized methods, which are shown in bold. All results are the average of 10 Monte Carlo runs.}
  \label{tab:noisydata}
  \centering  
  \begin{tabular}{ l || c c c | c c c | c c c}
\hline \hline
\multirow{2}{*}{Algorithm}
& \multicolumn{3}{c|}{$SNR = -10 dB$}   & \multicolumn{3}{c|}{$SNR = -5 dB$} & \multicolumn{3}{c}{$SNR = 5 dB$}   \\
\cline{2-10}
& RErr & Fit & Time    & RErr & Fit & Time  & RErr & Fit & Time \\ 
\cline{1-10} 	 	 	 	 	 	 	 	 	 	 	 	 	 	 	 	 
GR-SVD               & 8.92          & 0.33           & 1.62          & 17.69         & 2.03           & 4.17          & 46.75         & 22.75          & 8.89          \\
CS-SVD               & 8.92          & 0.31           & 1.64          & 17.68         & 2.11           & 5.87          & 47.07         & 22.21          & 11.09         \\
LS-SVD               & 8.92          & 0.30           & 6.44          & 17.68         & 2.07           & 14.48         & 47.80         & 21.00          & 33.74         \\
IS-SVD               & 8.91          & 0.36           & 2.16          & 17.59         & 2.56           & 4.81          & 43.71         & 27.77          & 8.69          \\
SRHT-SVD             & 8.92          & 0.28           & 70.87         & 17.68         & 2.08           & 188.01        & 46.43         & 23.27          & 447.71        \\

\hline
GRpi-SVD             & 5.25          & 41.33          & 13.97         & 2.72          & 84.94          & 12.54         & \textbf{1.00}          & 98.34          & 1.57          \\
GR3i-SVD             & 8.80          & 1.64           & 10.94         & 15.74         & 12.83          & 39.12         & 21.39         & 64.65          & 33.97         \\
GR2i-SVD              & 5.73          & 35.99          & 11.74         & 8.97          & 50.31          & 11.74         & 16.18         & 73.27          & 10.65         \\
GRrr-SVD              & 6.53          & 26.97          & 27.34         & 5.84          & 67.67          & 23.84         & 1.04          & 98.27          & 5.67          \\
\textbf{rBKI (ours)} & \textbf{1.01} & \textbf{88.73} & \textbf{2.00} & \textbf{1.00} & \textbf{94.46} & \textbf{1.68} & \textbf{1.00} & \textbf{98.35} & \textbf{1.47} \\

\hline
T-SVD                & 1.00          & 88.82          & 5.14          & 1.00          & 94.46          & 4.30          & 1.00          & 98.35          & 3.74    \\
\hline \hline
  \end{tabular}
\end{table*}

  \begin{table*}[htbp] \scriptsize
    \renewcommand\arraystretch{1.2} 
    \caption{Comparison in terms of RErr, Fit(\%) and Time(s) of variant methods on power functional data with fast singular value decay at different scales. In all cases, the level of Gaussian random noise is SNR=5dB. Compared to the deterministic T-SVD, our method achieve the best results among the randomized methods, which are shown in bold. All results are the average of 10 Monte Carlo runs.} 
    \label{tab:fundata}
    \centering  
    \begin{tabular}{ l || c c c | c c c | c c c}
    \hline \hline
    \multirow{2}{*}{Algorithm}
    & \multicolumn{3}{c|}{$N=3,\;I=200,\; R=10$}   & \multicolumn{3}{c|}{$N=3,\;I=500,\; R=25$} & \multicolumn{3}{c}{$N=4,\; I=30,\; R=3$}   \\
    \cline{2-10}
    & RErr & Fit & Time  & RErr & Fit & Time  & RErr & Fit & Time \\ 
    \cline{1-10}     	 	 	 	 	 	 	 	 	 	 	 	 	 	 	 	 
    GR-SVD               & 12.73         & 62.91          & 5.42          & 16.89         & 71.52          & 151.82         & 6.92          & 50.14          & 0.93          \\
  CS-SVD               & 12.64         & 63.20          & 5.55          & 16.77         & 71.72          & 138.88         & 7.58          & 45.34          & 1.17          \\
  LS-SVD               & 13.98         & 59.28          & 21.41         & 19.04         & 67.88          & 759.37         & 11.52         & 16.93          & 1.90          \\
  IS-SVD               & 10.78         & 68.60          & 3.86          & 14.63         & 75.33          & 156.48         & 3.54          & 74.48          & 0.75          \\
  SRHT-SVD             & 12.78         & 62.78          & 105.05        & 16.76         & 71.73          & 3320.75        & 7.50          & 45.93          & 10.20         \\
\hline
  GRpi-SVD             & 2.12          & 93.83          & 1.97          & 2.61          & 95.60          & 59.75          & 1.54          & 88.91          & 0.83          \\
  GR3i-SVD             & 20.58         & 40.05          & 28.03         & 18.57         & 68.68          & 683.73         & 13.85         & 0.18           & 0.56          \\
  GR2i-SVD              & 2.19          & 93.62          & 3.11          & 2.24          & 96.21          & 59.92          & 2.11          & 84.77          & 0.81          \\
  GRrr-SVD              & 2.53          & 92.63          & 9.14          & 2.94          & 95.05          & 105.42         & 1.89          & 86.34          & 3.58          \\
  \textbf{rBKI (ours)} & \textbf{1.10} & \textbf{96.80} & \textbf{0.85} & \textbf{1.10} & \textbf{98.14} & \textbf{43.51} & \textbf{1.00} & \textbf{92.79} & \textbf{0.11} \\
\hline
  T-SVD                & 1.00          & 97.09          & 2.48          & 1.00          & 98.31          & 84.73          & 1.00          & 92.79          & 0.17          \\
\hline \hline
    \end{tabular}
  \end{table*}

\begin{figure}[htb!]
  \centerline{
    \subfigure[Gaussian random data]{
      \includegraphics[width=0.25\textwidth,height=0.22\textwidth]{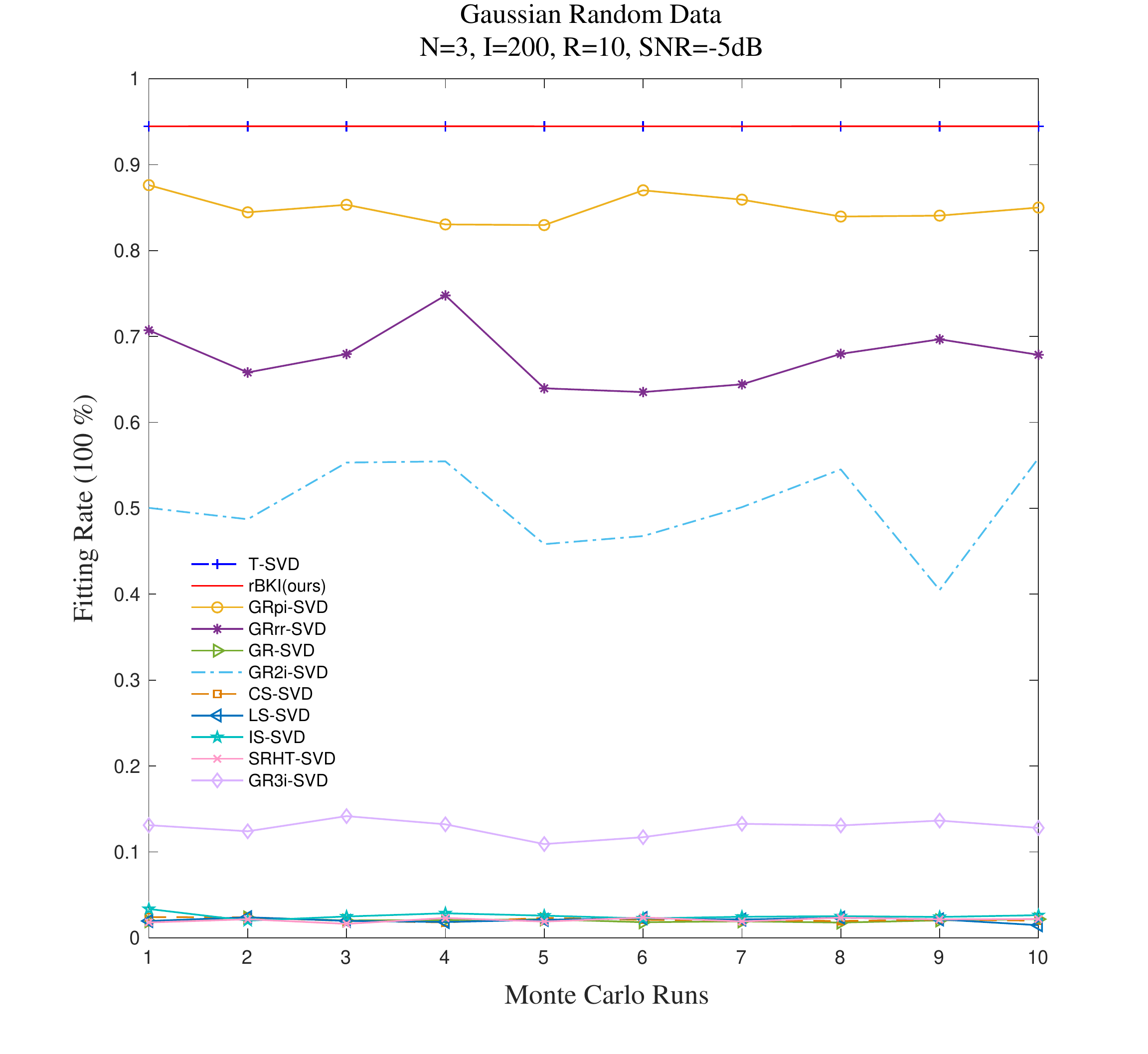}}\hfil
      \subfigure[Power functional data]{
      \includegraphics[width=0.25\textwidth,height=0.22\textwidth]{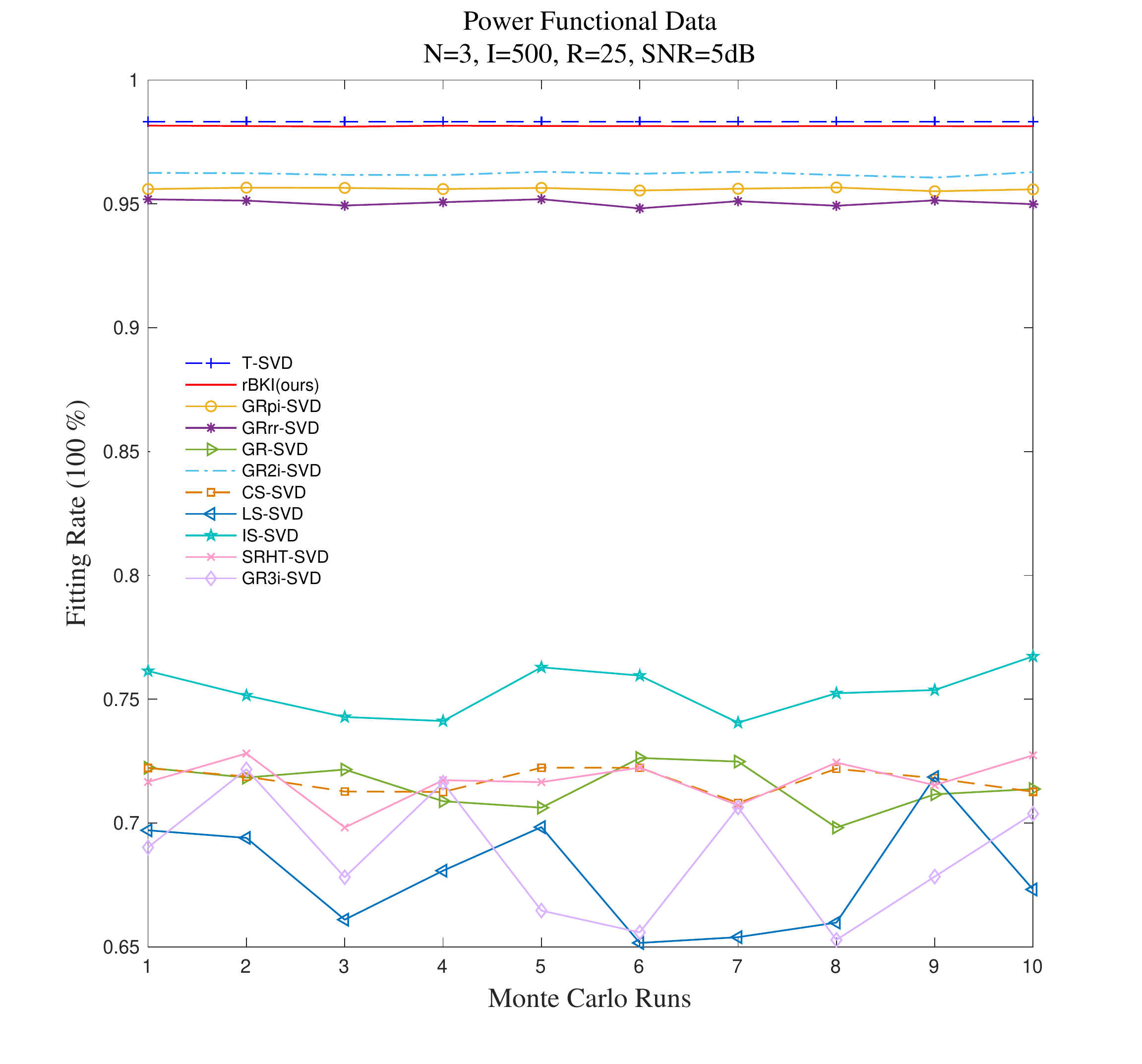}}
      }
  \caption{Comparison of variant methods in terms of fitting rate over 10 Monte Carlo runs on (a): an (N=3, I=200, R=10, SNR=-5dB) Gaussian random tensor and (b): an (N=3, I=500, R=25, SNR=5dB) power functional tensor.}
  \label{fig:Fitcurve}
\end{figure}

\begin{figure}[htb!]
  \centerline{
      \includegraphics[width=0.75\linewidth,height=0.6\linewidth]{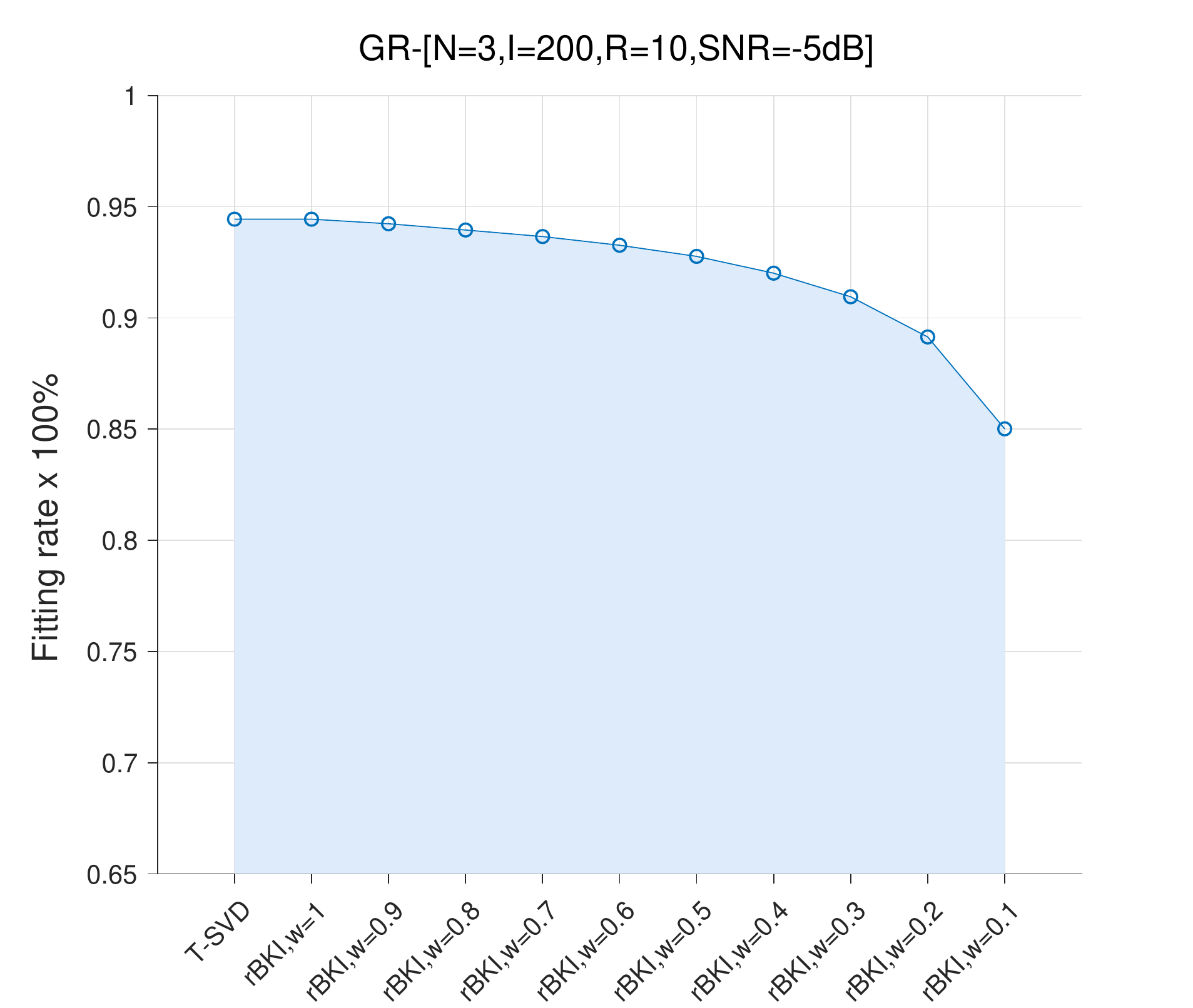}
  }
  \caption{Comparison of T-SVD and the proposed rBKI-TK method with different sampling rate $\omega = [0.1:0.1:1]$ in terms of fitting rate on an (N=3, I=200, R=10, SNR=-5dB) Gaussian random tensor.}
  \label{fig:samprate}
  \end{figure}

As shown, in general, our method performed consistently and optimally in terms of RErr and Fit in most cases, and was completely faster than the other methods.
Specifically, for results of the Gaussian random noisy data (in Table \ref{tab:noisydata}), as the noise ratio increased from SNR = 5dB to SNR = -10dB, non-iterative methods and other compared iterative methods were severely invalidated by noise interference and the Fit dropped rapidly, while our method remained valid and the Fit stayed above 85\%. 
Particularly, as the noise interference decreased, the Fit of some methods increased while the RErr became larger, indicating that these methods did not accurately capture the important information corresponding to large singular values. 
By using these two metrics, the validity of the compared methods can be measured more accurately.

Besides, for results of the power functional noisy data (in Table \ref{tab:fundata}), in the third-order tensor case, as the dimensionality increased from 200 to 500, our method achieved the best results comparable to T-SVD and is more time efficient, while other methods performed mediocre.
However, in the fourth-order tensor case, the Fit of certain methods (e.g., LS-SVD, GR3i-SVD)
decreased rapidly and exhibited poor scalability.

Furthermore, comparison in terms of Fit of compared methods over 10 Monte Carlo runs was shown in Fig. \ref{fig:Fitcurve}. 
Iterative methods (except GR3i-SVD) achieved consistent and excellent performances, while non-iterative methods performed poorly and inconsistently. Our method achieved the best result comparable to T-SVD.
Moreover, for an (N=3, I=200, R=10, SNR=-5dB) Gaussian random tensor, given partial data (even with a sampling rate $\omega=0.1$), our method also had good potential to capture an accurate subspace with a Fit of $85\%$, see Fig. \ref{fig:samprate}.

\subsubsection{Real-world Data}

In this simulation, three YUV video sequences\footnote{http://trace.eas.asu.edu/yuv/index.html} were converted to fourth order tensors (pixels $\times$ pixels $\times$ RGB $\times$ frames) and used for denoising evaluations, including ``flower-cif'' video sequences with a size of $288 \times 352 \times 3 \times 250 $, ``news-cif'' and ``hall-cif'' video sequences with a same size of $288 \times 352 \times 3 \times 300 $. 
To evaluate the performance of denoising, these video sequences were corrupted with different level of Gaussian random (GR) noise: a less noisy case of ``flower-cif'' video was corrupted by SNR = 10 dB GR-noise, and a noisier case of ``news-cif'' and ``hall-cif'' video were corrupted by SNR = 5dB GR-noise.
Besides, the denoising evaluation of ``flower-cif'' video was carried out on two principles: the ``one-pass'' processed a complete set (250 frames) directly, while the ``divide and conquer'' processed five subsets (50 frames) sequentially.
Moreover, with different rank setting, the full ``news-cif'' and ``hall-cif'' video were decomposed with a rank of $85 \times 90 \times 3 \times 30$, while the full and subsets of ``flower-cif'' video were decomposed with a rank of $85 \times 90 \times 3 \times 50$ and $85 \times 90 \times 3 \times 10$, respectively.

For comparison, the benchmark T-SVD and six competitive methods selected from the above compared methods, i.e., three iterative methods: GRpi-SVD, GRrr-SVD, GR2i-SVD, and three non-iterative methods: GR-SVD, CS-SVD, IS-SVD, were adopted in simulations.
The average results in terms of PSNR (dB), Fit (\%) and Time (s) were given in Table \ref{tab:denoising}.
In addition, a visual comparison of denoising performance was shown in Fig. \ref{fig:denoising}, where the 5th-frame of ``flower'' video, the 123th-frame of ``news'' video and the 120th-frame of ``hall'' video were selected for demonstration.

\begin{table*}[!t] \scriptsize
  \renewcommand\arraystretch{1.2} 
  \caption{Comparison of the denoising performance in terms of PSNR(dB), Fit(\%) and Time(s) of variant methods on three YUV video data. Compared to the deterministic T-SVD, our method achieve the best results among the randomized methods, which are shown in bold. All results are the average of 10 Monte Carlo runs.}
  \label{tab:denoising}
  \centering  
  \begin{tabular}{ l || c c c | c c c | c c c | c c c }
    \hline
    \multirow{2}{*}{Algorithm}
    & \multicolumn{3}{c|}{news 5dB (one-pass)}   & \multicolumn{3}{c|}{hall 5dB (one-pass, $\omega$=0.5)}  & \multicolumn{3}{c|}{flower 10dB (one-pass)		}   & \multicolumn{3}{c}{flower 10dB (divide-conquer)}\\
    \cline{2-13}
    & PSNR & Fit & Time  & PSNR & Fit & Time  & PSNR & Fit & Time  & PSNR & Fit & Time\\ 
    \cline{1-13} 	 	 	 	 	 	 	 	 	 	 	 	 	 	 	 	 
    GR-SVD               & 20.34          & 68.25          & 127.02         & 20.10          & 76.92          & 63.44          & 18.01          & 78.11          & 37.15          & 18.44          & 80.08          & 97.39          \\
    CS-SVD               & 20.30          & 68.24          & 130.72         & 20.32          & 77.15          & 75.49          & 18.05          & 78.15          & 29.18          & 18.48          & 80.21          & 100.21         \\
    IS-SVD               & 20.58          & 69.69          & 86.95          & 20.47          & 77.81          & 36.89          & 18.02          & 78.28          & 34.70          & 18.45          & 80.49          & 52.02          \\
  \hline
    GRpi-SVD             & 25.22          & 83.27          & 32.46          & 25.12          & 87.82          & 25.66          & 19.92          & 83.94          & 27.24          & 20.64          & 85.89          & 28.44          \\
    GR2i-QR              & 23.83          & 80.33          & 35.18          & 24.46          & 86.70          & 21.62          & 19.24          & 81.76          & 18.19          & 19.63          & 83.83          & 34.41          \\
    GRrr-QR              & 24.88          & 82.57          & 90.06          & 25.32          & 88.12          & 55.92          & 19.83          & 83.63          & 54.17          & 20.39          & 85.46          & 62.21          \\
    \textbf{rBKI (ours)} & \textbf{28.69} & \textbf{88.97} & \textbf{18.50} & \textbf{28.27} & \textbf{91.54} & \textbf{16.30} & \textbf{20.77} & \textbf{85.95} & \textbf{14.81} & \textbf{21.56} & \textbf{87.56} & \textbf{14.30} \\
  \hline
    T-SVD                & 28.73          & 89.03          & 51.85          & 28.50          & 91.81          & 38.26          & 20.80          & 85.98          & 40.79          & 21.57          & 87.58          & 30.27  \\
  \hline \hline
    \end{tabular}
\end{table*}

\begin{figure*}[!t]
  \centerline{
    \includegraphics[width=1\linewidth,height=0.35\linewidth]{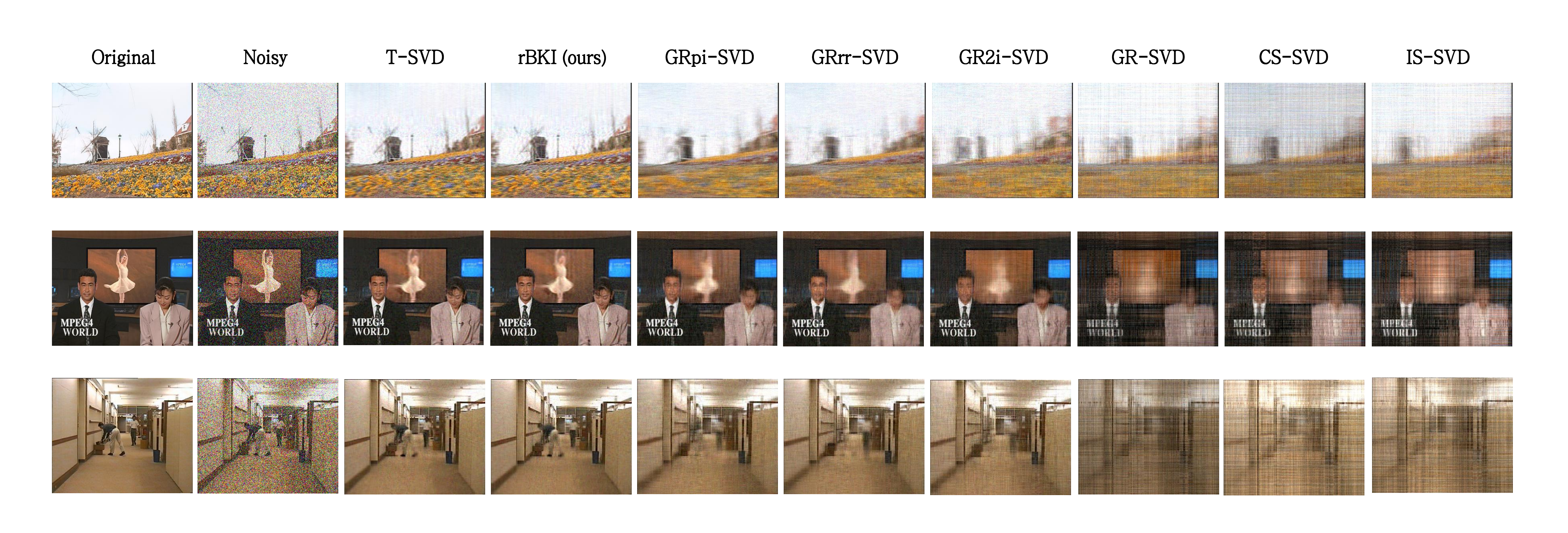}
  }
  \caption{Examples of video denoising, including original data, noisy data and denoised data. The 5th-frame of ``flower'' video (top), the 123th-frame of ``news'' video (middle) and the 120th-frame of ``hall'' video (bottom) are chosen for demonstration.}
  \label{fig:denoising}
  \end{figure*}

  As shown in Table \ref{tab:denoising}, the ``divide-conquer'' principle slightly outperformed the ``one-pass'' principle in terms of efficiency and accuracy for the same rank-to-size ratio (R/I), probably due to the stronger low-rankness of adjacent video frames.
  Thus, for specific video denoising tasks, the divide-and-conquer principle can be considered to find a better trade-off between time and accuracy.
  Specifically, in the less noisy case of ``flower'' video (SNR=10dB), the denoising performance of our method (rBKI) was comparable to T-SVD in that it provided clearer details of flowers and buildings, while the results of other methods are blurred.
  Besides, when denoising the noisier ``news'' and ``hall'' video (SNR=5dB): for the ``news'' video, iterative methods produced a clear foreground (i.e. the letters "MPEG4 WORLD") compared to non-iterative methods, and our method (rBKI) further produced a clearer background (i.e. the ballet dancer) and the expression of presenters; for the ``hall'' video, T-SVD and our method (with a sampling rate $\omega$ of 0.5) can effectively denoise and accurately recover moving objects in the hall, while other methods failed.
  
  \subsection{Simulations on rBKI-TK-TR}  
  In this subsection, two large-scale image datasets were used to evaluate the performance of compared methods on data compression and recovery. 
  One is the COIL-100 dataset of size $32 \times 32 \times 3 \times 7200 $, another is a subset of CIFAR-10 dataset of size  $32 \times 32 \times 3 \times 10000 $, and both datasets contain a number of elements at the $10^7$ level. 
  The proposed rBKI-TK-TR model were compared with two existing hierarchical TR models based on randomized sketching techniques, i.e., Gaunssian random TK-TR model (GR-TK-TR) \cite{Yuan2019TR} and rank revealing TK-TR model (RR-TK-TR) \cite{rrQR2021}.
  
  To evaluate the compression efficiency and recovery accuracy of comparison models, the TR-rank was set ranged from R= [3,3,3,3], R=[6,6,3,6] to R=[9,9,3,9], while the prior Tucker compress-size was set to $\min (5R,I_n)$.
  Note that in order to minimize information loss without losing efficiency, the multilinear rank, sketch-size and compress-size were set to the same value in the prior TKD. 
  All results in Table \ref{tab:realdata} were the average of 10 Monte Carlo runs. 
  Moreover, Fig. \ref{fig:TR-Fit} and Fig. \ref{fig:TR-Time} provided the comparison in terms of fitting rate and running time over 10 Monte Carlo runs, respectively.

  \begin{table*}[!t] \scriptsize
  \renewcommand\arraystretch{1.2} 
  \caption{Comparison in terms of Fit(\%) and Time (s) of variant hierarchical TR models on two large-scale image datasets. The TR-rank is set ranged from R=[3,3,3,3], R=[6,6,3,6] to R=[9,9,3,9]. Compared to the deterministic TR-ALS, our method achieve the best results among the randomized methods, which are shown in bold. All results are the average of 10 Monte Carlo runs.}
  \label{tab:realdata}
  \centering  
  \begin{tabular}{ l || c c | c c | c c | c c | c c | c c }
    \hline \hline
    \multirow{3}{*}{Algorithm}
    & \multicolumn{6}{c|}{COIL-100}  & \multicolumn{6}{c}{CIFAR-10}\\
    \cline{2-13}
     & \multicolumn{2}{c|}{R = [3, 3, 3, 3]}   & \multicolumn{2}{c|}{R = [6, 6, 3, 6]}   & \multicolumn{2}{c|}{R = [9, 9, 3, 9]}   & \multicolumn{2}{c|}{R = [3, 3, 3, 3]}   & \multicolumn{2}{c|}{R = [6, 6, 3, 6]}   & \multicolumn{2}{c}{R = [9, 9, 3, 9]}\\
    \cline{2-13}
     & Fit & Time  & Fit & Time  & Fit & Time & Fit & Time  & Fit & Time  & Fit & Time\\ 
    \cline{1-13} 
    GR-TK-TR      & 57.33          & 15.40         & 65.25          & 19.79         & 69.20          & 25.14         & 64.94          & 23.62         & 71.03          & 29.28         & 73.63          & 30.51         \\
    RR-TK-TR      & 64.63          & 12.90         & 72.60          & 16.15         & 76.04          & 20.14         & 71.78          & 19.78         & 76.77          & 24.21         & 78.99          & 24.88         \\
    \textbf{Ours} & \textbf{64.68} & \textbf{4.27} & \textbf{72.69} & \textbf{5.51} & \textbf{76.25} & \textbf{7.23} & \textbf{71.80} & \textbf{7.11} & \textbf{76.91} & \textbf{8.96} & \textbf{79.00} & \textbf{9.38} \\
\hline 
    TR-ALS        & 64.85          & 111.69        & 73.80          & 277.18        & 77.78          & 474.53        & 71.91          & 156.68        & 77.46          & 390.90        & 80.00          & 802.87          \\
\hline 
  \end{tabular}
\end{table*}

\begin{figure}[htb!]
  \centerline{
  \subfigure[COIL-100]{
  \includegraphics[width=0.25\textwidth,height=0.2\textwidth]{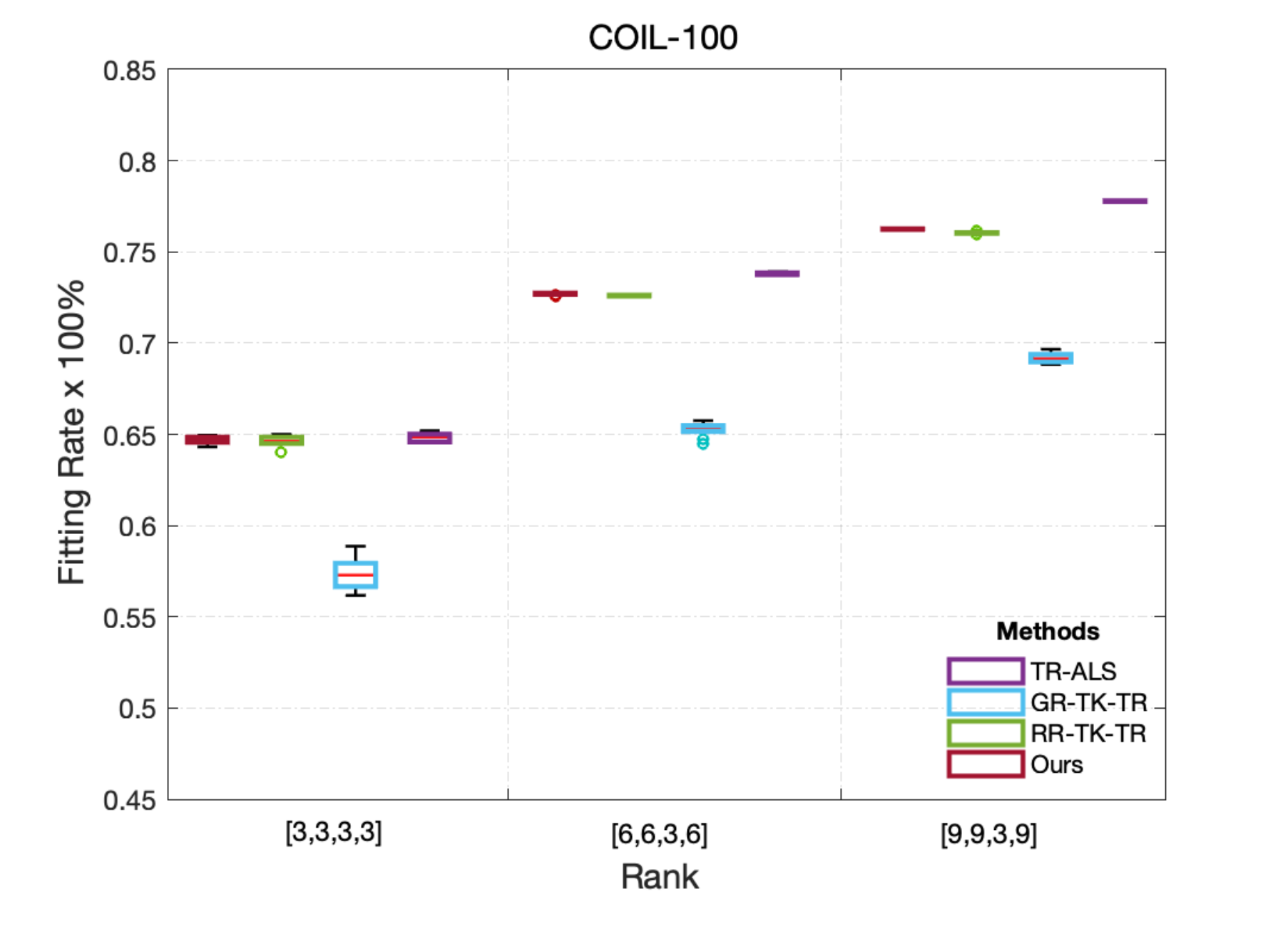}}\hfil
  \subfigure[CIFAR-10]{
  \includegraphics[width=0.25\textwidth,height=0.2\textwidth]{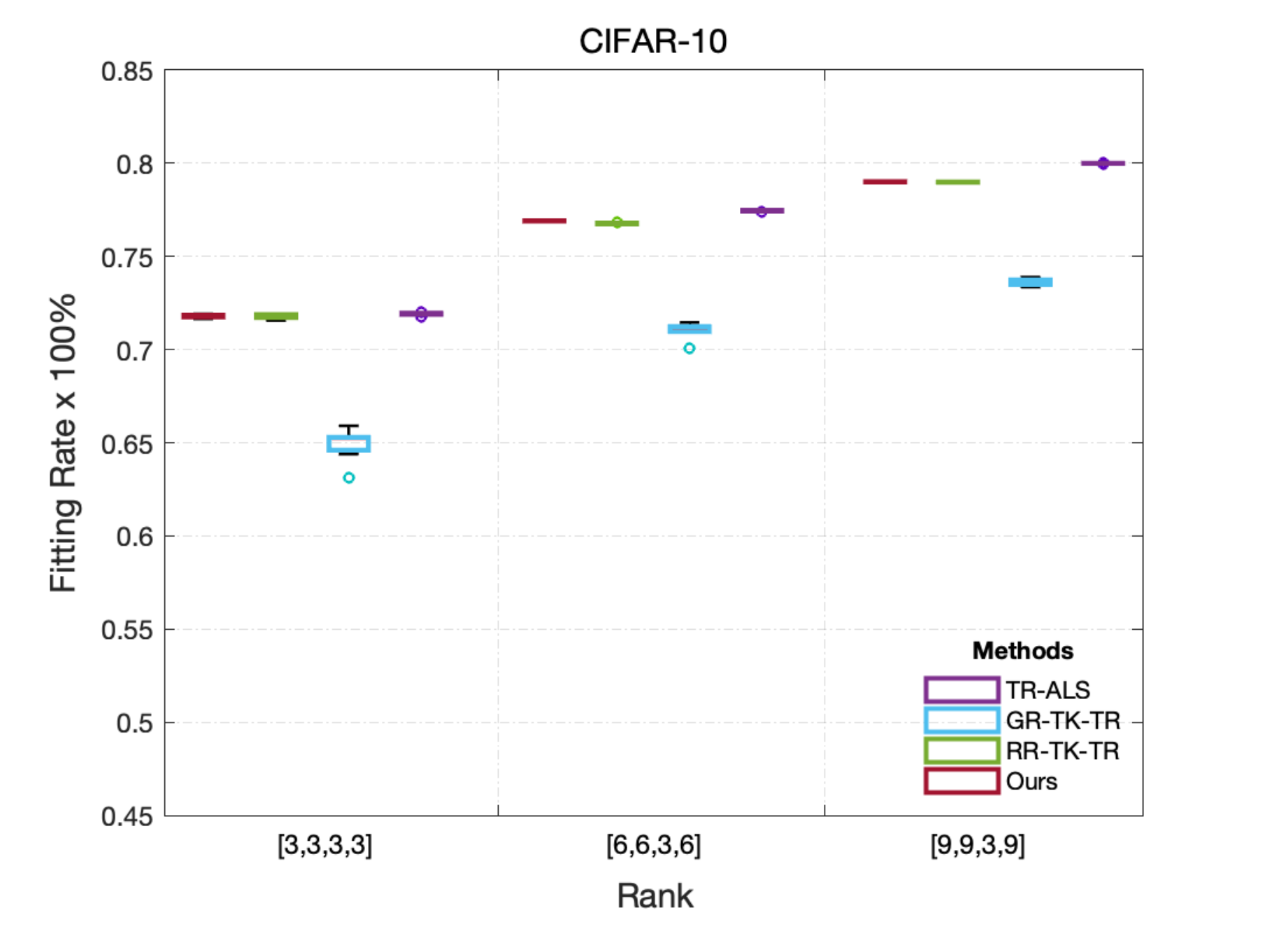}}
  }
  \caption{ Comparison in terms of Fit ($\times 100\%$) of the compared methods on (a) COIL-100 dataset and (b) CIFAR-10 dataset. The TR-rank is set ranged from R=[3,3,3,3], R=[6,6,3,6] to R=[9,9,3,9]. }
  \label{fig:TR-Fit}
\end{figure}

\begin{figure}[htb!]
  \centerline{
  \subfigure[COIL-100]{
  \includegraphics[width=0.25\textwidth,height=0.2\textwidth]{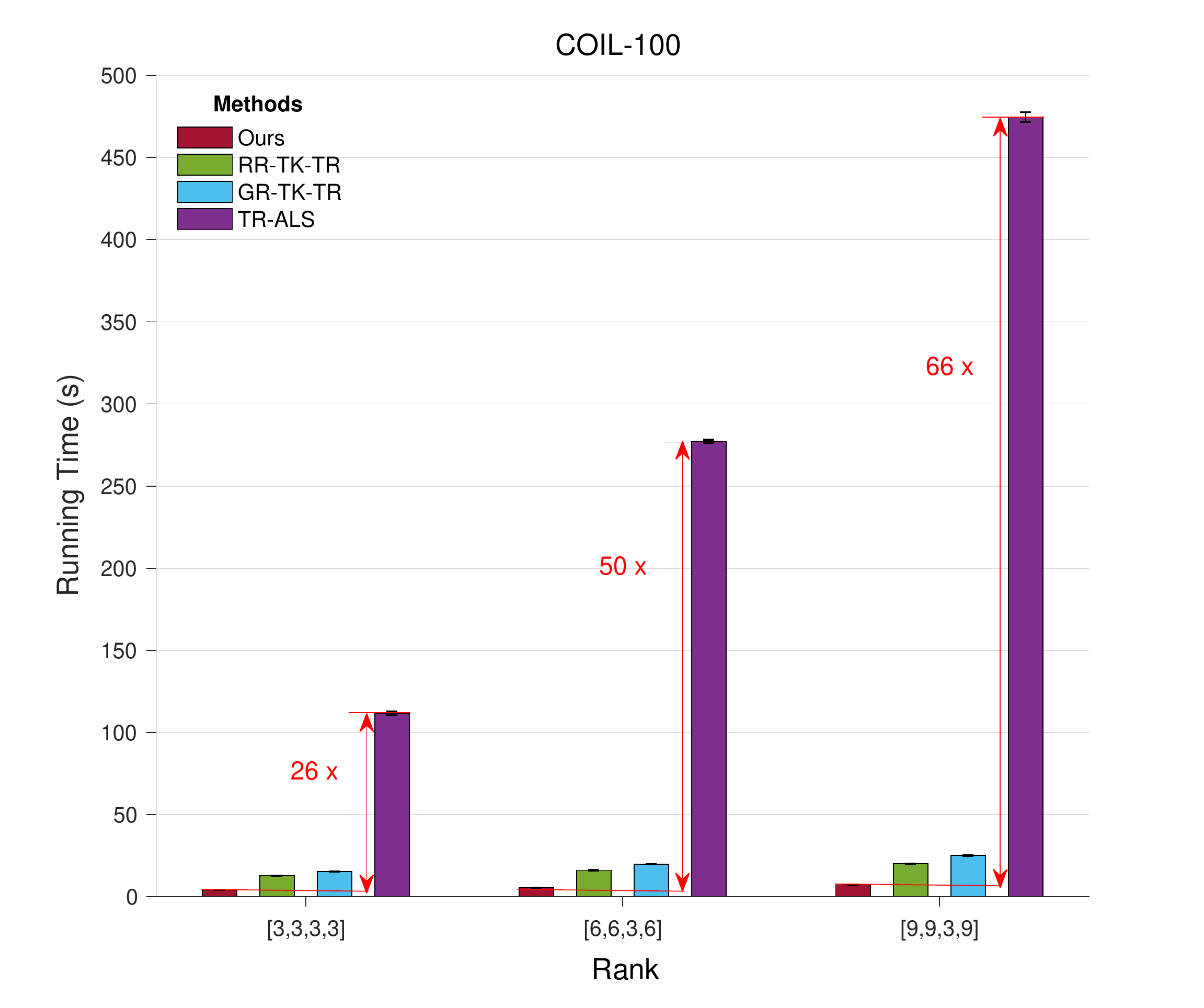}}\hfil
  \subfigure[CIFAR-10]{
  \includegraphics[width=0.25\textwidth,height=0.2\textwidth]{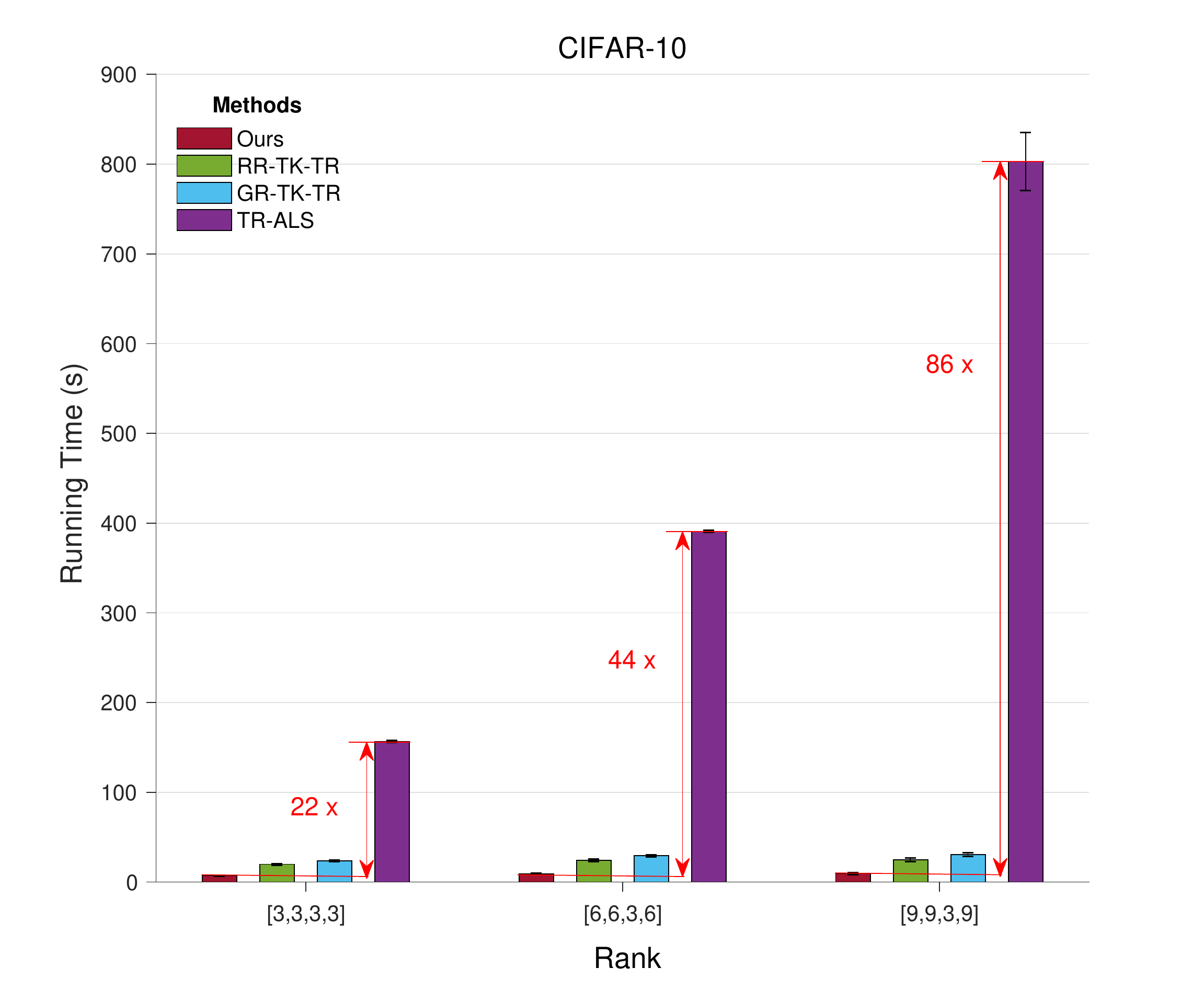}}
  }
  \caption{Time consumption comparison of the compared methods on (a) COIL-100 dataset and (b) CIFAR-10 dataset. The TR-rank is set ranged from R=[3,3,3,3], R=[6,6,3,6] to R=[9,9,3,9]. }
  \label{fig:TR-Time}
\end{figure}

As shown (in Table \ref{tab:realdata}), our method achieved the best recovery accuracy (in terms of Fit) comparable to the TR-ALS while more time efficient than other compared methods.  
In all cases, the fitting rate of the GR-TK-TR was lower than the other methods and slightly unstable, while iterative methods, i.e., RR-TK-TR and ours performed stably and better than the non-iterative GR-TK-TR (see Fig. \ref{fig:TR-Fit}).
Besides, the consumed time by the TR-ALS was extremely larger than the other randomized methods. Our method was the most time efficient, and as the rank increases the efficiency became more prominent, even 86 times faster than TR-ALS (see Fig. \ref{fig:TR-Time}).

\section{Conclusion}
\label{sec:conclusion}

In this work, an rBKI-based Tucker decomposition (rBKI-TK) is proposed for efficient denoising. 
The core idea is to obtain a more accurate projection subspace by rBKI, thus facilitating a higher accuracy approximation.
Theoretical analysis presents that the proposed rBKI-TK can achieve quasi-optimal approximation as the deterministic HOSVD in a more time-efficient fashion.
Besides, the rBKI-TK is also extended into a hierarchical TRD (rBKI-TK-TR) for efficient compression of large-scale data.
Experimental results promisingly validate that, on the one hand, the proposed rBKI-TK outperforms existing methods with excellent denoising performance in both synthetic and real-world video data (even with severe noise interference); on the other hand, the rBKI-TK-TR is more time efficient than existing methods, especially 86 times faster than the classical TR-ALS.

With developments of randomized sketching techniques, randomized tensor decomposition can be extended to other randomized methods and other linear algebra paradigms. 
Therefore, innovations and new insights in randomized tensor decomposition and its applications will be investigated in future work.

\section*{Acknowledgments}
This work was supported in part by the National Natural Science Foundation of China under Grant 62073087 and Grant U1911401.

\bibliographystyle{IEEEtran}

\bibliography{rBKI_TK_ycq.bib}

\end{document}